\documentclass{tlp}

\usepackage{amsmath}
\usepackage{graphicx}
\usepackage{multirow}

\usepackage{mathptmx}

\usepackage{times}
\usepackage{soul}
\usepackage{url}
\usepackage{amsmath}
\usepackage{booktabs}
\usepackage{algorithm}
\usepackage{algorithmic}
\usepackage{tabularx}

\usepackage{times, helvet, courier}
\usepackage{graphicx, epsfig, epstopdf}
\usepackage{wrapfig, caption, multirow, url}
\usepackage{mathptmx}
\usepackage{hyperref}
\usepackage{graphicx}
\usepackage{picinpar}
\usepackage{xcolor}
\usepackage{url}
\usepackage{csquotes}
\usepackage{listings}
\usepackage{algorithm}
\usepackage{algorithmic}
\usepackage{todonotes}
\usepackage{multirow} 
\usepackage{todonotes}

\newtheorem{example}{Example}

\newtheorem{definition}{Definition}

\newtheorem{proposition}{Proposition}

\newcounter{rownumber}[figure] 
\usepackage{amsfonts, amsmath, amssymb, mathrsfs}

\newcommand{\specialcellleft}[2][l]{%
        \begin{tabular}[#1]{@{}l@{}}#2\end{tabular}}
\newcommand{\specialcell}[2][c]{%
        \begin{tabular}[#1]{@{}c@{}}#2\end{tabular}}

\newcommand{\causes}{\mbox{ causes }}

\def\naf{\:{not}\:}
\def\lo{\left(}
\def\lc{\right)}
\def\cpsf{{CPSF}}

\def\prob{{\mathbf{success\_with}}}
\def\naf{{ not \:}}
\def\iif{\hbox{\bf if } \; }

\def\causes{\hbox{\bf causes}}

\def\executable{\hbox{\bf executable}}

\def\after{\;\hbox{\bf after} \;}

\usepackage{listings}
\lstdefinelanguage{clingo}{
        keywordstyle=[1]\usefont{OT1}{cmtt}{m}{n},%
        keywordstyle=[2]\textbf,%
        keywordstyle=[3]\usefont{OT1}{cmtt}{m}{n},
        alsoletter={\#,\&},%
        keywords=[1]{not,from,import,exists,if,else,return,while,
                break,and,or,for,in,del,and,class,subClass,concern,aspect,subCo,prop,rdf,cpsf,addBy,suppBy,neg,d,relation,holds,h,obs,action,fluent,occurs,req,group,leadTo,active,deg,comp,order,hSubCo,llh_sat_sub,llh_sat,llh_sat_sub_aux,step,deg_pos,nAllPosCon,nActPosCon,scoreLoS,wBus,wHum,wTru,wFun,wTim,wBou,wLif,wCom,wDat,last,addFun,func,formula,sat_formula,positiveImpact,conflict,member,sa_action,exec,sc_concern,sat_formulas,pos,most,least,twcp,twcn,tw_p,active,wLoS,availablePatch,instance,isInstanceOf,decomp_func,RDFtriple,rdf,rdfs,owl,addConcern,disjunction,conjunction,r,nPos,tw,higher,equal,not_highestTW,not_lowestTW,patch},%
        keywords=[2]{\#const,\#show,\#minimize,\#maximize,\#base,\#theory,\#count,\#external,\#program,\#script,\#end,\#heuristic,\#edge,\#project,\#show,\#sum,\#min,\#max},%
        keywords=[3]{&,&dom,&sum,&diff,&show,&minimize},%
        morecomment=[l]{\#\ },%
        morecomment=[l]{\%\ },%
        commentstyle={\color{darkgray}}%
}
\newcommand{\memo}[1]{
        \ifthenelse {\boolean{includeMemo}}{\medskip\noindent\fbox{\begin{minipage}[b]{\dimexpr\linewidth-1em}#1\end{minipage}}\medskip\newline} 
}

\begin{document}

\lefttitle{Thanh H. Nguyen, et al.}
\jnlPage{1}{8}
\jnlDoiYr{2021}
\doival{10.1017/xxxxx}

\title[Specifying and Reasoning about CPS through the Lens of the NIST CPS Framework]{Specifying and Reasoning about CPS through the Lens of the NIST CPS Framework}

\begin{authgrp}
\author{\sn{Thanh} \gn{Hai Nguyen}, \sn{Matthew} \gn{Bundas}, \sn{Tran} \gn{Cao Son}}
\affiliation{Department of Computer Science, New Mexico State University, Las Cruces, USA}
\affiliation{\email{thanhnh@nmsu.edu,bundasma@nmsu.edu,stran@nmsu.edu}}
\author{\sn{Marcello} \gn{Balduccini}, \sn{Kathleen} \gn{Campbell Garwood}}
\affiliation{Saint Joseph's University, Philadelphia, USA}
\affiliation{\email{mbalducc@sju.edu,kcampbel@sju.edu}}
\author{\sn{Edward} \gn{R. Griffor}}
\affiliation{National Institute of Standards and Technologies, USA}
\affiliation{\email{edward.griffor@nist.gov}}
\end{authgrp}

\history{\sub{xx xx xxxx;} \rev{xx xx xxxx;} \acc{xx xx xxxx}}

\maketitle

\begin{abstract}
This paper introduces a formal definition of a Cyber-Physical System (CPS) in the spirit of the CPS Framework proposed by the National Institute of Standards and Technology (NIST). It shows that using this definition, various problems related to concerns in a CPS can be precisely formalized and implemented using Answer Set Programming (ASP). These include problems related to the dependency or conflicts between concerns, how to mitigate an issue, and what the most suitable mitigation strategy for a given issue would be. It then shows how ASP can be used to develop an implementation that addresses the aforementioned problems. The paper concludes with a discussion of the potentials of the proposed methodologies.  
\end{abstract}


\section{Introduction} 
The utility (potable water, wastewater) distribution systems, the electric power grid, the transportation network, automated driving systems (ADS), hospital robots, and smart-home systems are a few examples of cyber-physical systems (CPS)\footnote{For brevity, we use CPS to stand for both the plural and the singular cyber-physical system.} that are (or soon to be) a part of our daily life. Before any CPS is deployed into the real-world, several concerns need to be investigated and addressed, e.g.,  why should someone  trust that the CPS will perform its functions safely, securely and reliably? How will such a system respond to a certain critical conditions and will that response be acceptable? In other words, evidence must be gathered and argued to be sufficient to conclude that critical properties of a CPS have been assured before its deployment. For financial and practical reasons, the validation and verification of a CPS should be done as early as possible, starting with its design. CPS are complex systems that evolve with use, requiring a principled methodology and tools for developing an assurance case before release to the market. Such a methodology and the tools for applying it are two key contributions of this paper. We present here a formalization of a CPS with a clearly defined semantics that enables the assessment of critical system properties.  The need for such a foundation for assurance can be seen in the next example. 

\begin{example} 
        \label{example:conflict_type_1}
        Suppose that we would like to develop an Automated Driving System (ADS). We have two constraints that we would like to enforce: ({\em a}) packets sent from the wind-sensor, a part of the  situational awareness module ({\tt SAM}), to the main processor must be fast and reliable;  
        ({\em b}) all communication channel must be encrypted. We will refer to ({\em a}) and ({\em b}) as an {\tt \small Integrity} concern and 
        {\tt \small Encryption} concern, respectively.
        
        Consider a situation in which the ADS has only one possible communication channel, which is fast, reliable when encryption is disabled, but is not when encryption is enabled. In this situation, the two constraints are in conflict with each other. It is impossible to satisfy both of them. 
        
        Assume that we also have some preference, called {\tt \small Verification}, which is related to the verification of received data. Encrypted data would have been preferred to non-encrypted one. If the wind-sensor uses the non-encrypted socket communication, it can satisfy (or \emph{positively} affect) the {\tt \small Integrity} concern but it does not satisfy (or \emph{negatively} affect) the {\tt \small Verification} preference. 
\end{example}

In this paper, we view a CPS as a dynamic system that consists of several components with various constraints and preferences which will be referred as \emph{concerns} hereafter. Given a concrete state of the system, a concern might or might not be satisfied. We aim at laying the mathematical foundation for the study of CPS' concerns. This foundation must allow CPS  developers and practitioners to represent and reason about the concerns and answer questions such as ({\em i}) will a certain concern or a set of concerns be satisfied? ({\em ii}) is there any potential conflict between the concerns? and ({\em iii}) how can we generate the best plan that addresses an issue raised by the lack of satisfaction of a concern? Readers familiar with research in representing and reasoning about dynamic systems might wonder whether well-known formalisms for representing and reasoning about dynamic systems such as automata, action languages, Markov decision process, etc. could be used for this purpose. Indeed, our proposed framework extends these formalisms by adding a layer for modeling the components and concerns in CPS.

To achieve our goal, we propose a formalism for representing and reasoning about concerns of CPS. We will focus on the properties described in the CPS Framework (\cpsf{}) proposed by the CPS Public Working Group (CPS PWG) organized by the National Institute of Standards and Technology (NIST)  \cite{Griffor2017FrameworkFC_vol1,Griffor2017FrameworkFC_vol2,Wollman2017FrameworkFC}. This framework defines several important concepts related to CPS such as \textit{facets} (modes of the system engineering process: conceptualization, realization and assurance), \textit{concerns} (areas of concern), and \textit{aspects} (clusters of concerns: functional, business, human, trustworthiness, timing, data, composition, boundaries, and lifecycle). These concepts are organized in an ontology which is easily extensible and allows us to better manage development and implementation within, and across, multiple application domains. We formally propose the notion of a CPS system that ({\em i}) considers constraints among concerns; 
({\em ii}) enables the automatic identification of conflicts between concerns; and (\emph{iii}) enables the application of planning techniques in computing mitigation strategies. Building and establishing upon {\cpsf} are important properties of our research, which distinguish it from much of the work done on CPS so far. While most of the prior research is focused on a specific class of CPS or of aspects, e.g., CPS for smart grids or concerns related to cybersecurity \cite{bau19}, the methodology we provide is intentionally domain-independent and applicable to any class of CPS.

The paper is organized as follows. Section~\ref{section:background} presents a brief overview of the CPS framework, answer set programming, action language, and reasoning with ontologies using answer set programming. Section~\ref{section:cps_theory} contains the main contribution of the paper, a formalization of a CPS theory, which includes a specification of CPS domain and the semantics defining when a concern is satisfied. It also formally defines several reasoning tasks related to the satisfaction of concerns such as ({\em i}) when is a concern satisfied; ({\em ii}) what are the most/least trustworthy components of a CPS system; ({\em iii}) is the CPS system compliant; ({\em iv}) computing a mitigation strategy for a system when some concerns become unsatisfied; ({\em v}) which mitigation strategy has the best chance to succeed. Section~\ref{section:implementation} provides an answer set programming implementation of the tasks. The paper concludes with the discussion of the related work. The paper is arranged in a way such that it can be of interest to different groups of readers. Specifically, it separates the formal definitions of a CPS, and the reasoning tasks associated with it, from a concrete implementation of the reasoning tasks. As such, a reader only interested in the formal theories would likely be interested in Section~\ref{section:cps_theory}. On the other hand, the code in Section~\ref{section:implementation} would be of interest to readers who would like to experiment with their own CPS.    

%
%
\section{Background}
\label{section:background}

This section reviews the background notions that will be used in the paper, including the CPS ontology, answer set programming, and the use of logic programming in ontology reasoning. 

\subsection{NIST CPS\ Framework and the CPS Ontology}
\label{sub-section:bg_CPS_ontology}

One of the major challenges in designing, maintaining and operating CPS is the diversity of areas of expertise involved in these tasks, and in the structure of the CPS itself. For example, developing a ``smart ship'' \cite{mos01} involves close interaction among, and cooperation of, experts in disciplines ranging from cybersecurity to air conditioning systems and from propulsion to navigation. As demonstrated by, e.g., NASA's Mars Climate Orbiter\footnote{https://www.simscale.com/blog/2017/12/nasa-mars-climate-orbiter-metric/}, ensuring a \textit{shared understanding} of a CPS and the interoperability of its components is an essential step towards its success -- a goal that is made even more elusive by the fact that the areas of knowledge relevant to a CPS vary greatly depending to the type of CPS considered. 

For this purpose, NIST recently hosted a Public Working Group on CPS with the aim of capturing input from those involved in CPS to define a \emph{CPS reference framework} supporting common definitions and facilitating interoperability between such systems, regardless of the type of CPS considered. A key outcome of that work was the CPS Framework (Release 1.0, published as three separate NIST Special Publications \cite{Griffor2017FrameworkFC_vol1,Griffor2017FrameworkFC_vol2,Wollman2017FrameworkFC}), which proposes a means of describing three \textit{facets} during the life of a CPS: conceptualization, realization, and assurance of CPS; and to facilitate these descriptions through analytical lenses, called \textit{aspects}, which group common concerns addressed by the builders and operators of the CPS. The CPS Framework 
articulates the artifacts of a CPS in a precise way, including the concerns 
that motivate important requirements to be considered in 
conceptualizing, realizing (including operating), and assuring CPS. Albeit helpful, being a reference framework the CPS Framework only helps with the specification of a CPS and the discussion among experts. It does not, by itself, 
reduce the amount of work necessary to analyze the CPS\ and its evolution of the CPS lifecycle. 

This realization gave impulse to the investigation that ultimately resulted in the \emph{CPS Ontology} \cite{Balduccini2018OntologyBasedRA,Nguyen2020ReasoningAT}, which provides a CPS analysis methodology based on the \emph{CPS Framework} featuring a vocabulary that describes and supports the understanding and development of new and existing CPS, including those designed to interact with other CPS and function in multiple interconnected infrastructure environments. 

\begin{figure*}[htbp]
        \centering
        \includegraphics[width=1.0\textwidth]{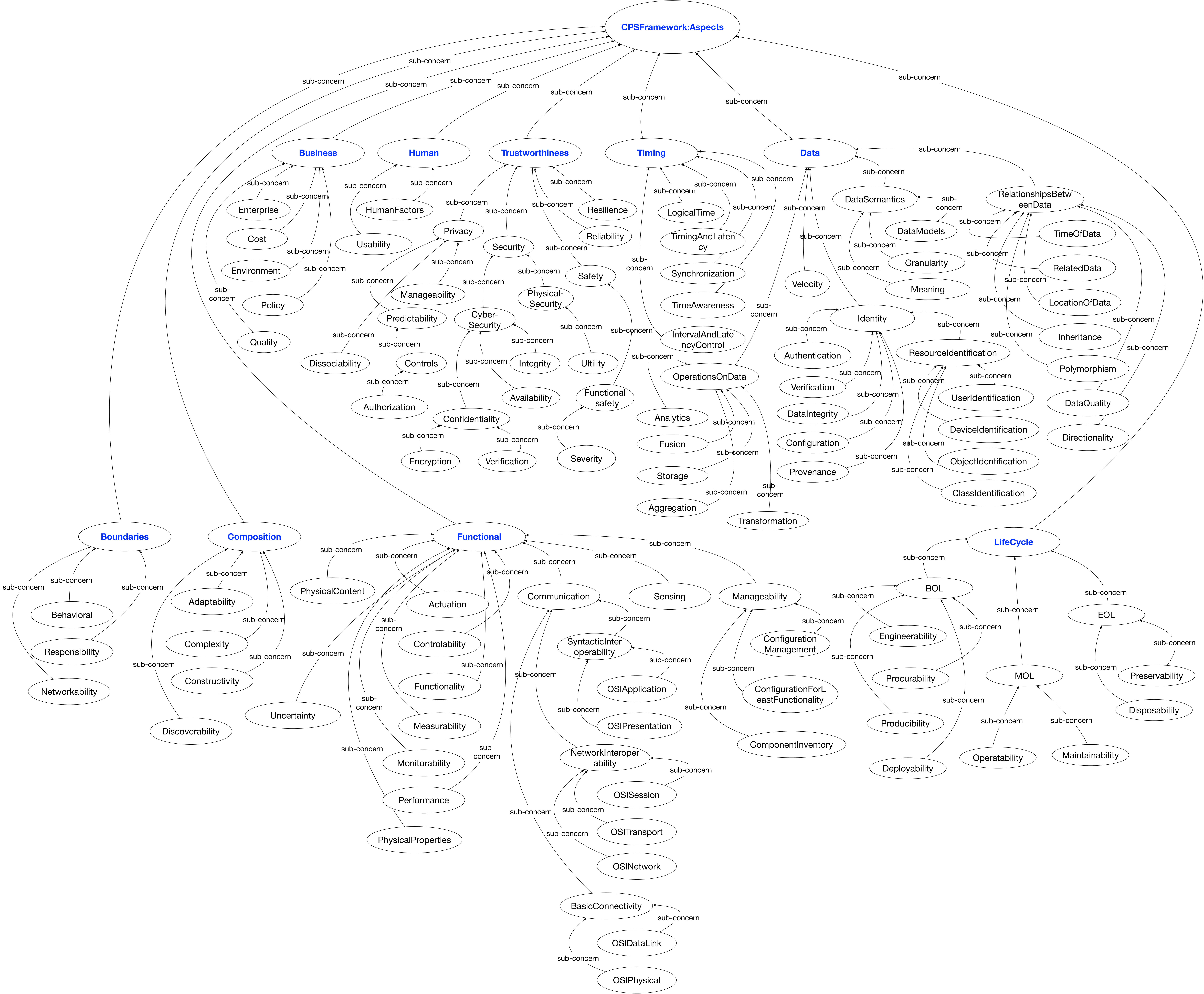}
        \caption{NIST CPS Ontology\label{fig:cps_ontology}}
\end{figure*}

At the core of the CPS Framework and of the CPS\ Ontology are the notions of domains, facets (conceptualization, realization and assurance), aspects and concerns, and a cyber-physical functional decomposition. The product of the conceptualization facet is a model of the CPS (requirements added to address prioritized concerns), the product of the realization facet is a CPS satisfying the model and the product of the assurance facet is assurance case for the prioritized set of concerns. \emph{Domains} represent the different application areas of CPS such as automated driving systems, electrical grid, etc. \emph{Concerns} are characteristics of a system that one or more of its stakeholders are concerned about. They are addressed throughout the lifecycle of a CPS, including development, maintenance, operation and disposal.
\emph{Requirements} are assertions about the state variables of a CPS aimed at addressing the concerns. The reader should note that, in line with the current {\cpsf} specification, we consider the term \emph{property} to be a synonym of requirement, and we use the two terms interchangeably in the rest of this paper.
\emph{Artifacts} are the elements of products of the facets for a CPS and include requirements, design elements, tests, and judgments.
\emph{Aspects} are the ten high-level concerns of the CPS Framework: functional, business, human, trustworthiness, timing, data, communication, boundaries, composition, and lifecycle.
        \begin{itemize}
                \item \emph{Functional} aspect is a set of concerns related to the sensing, computational, control, communications and actuation functions of the CPS.
                \item \emph{Business} aspect includes the concerns about enterprise, time to market, environment, regulation, cost, etc.
                \item \emph{Human} aspect is a set of concerns related to how a CPS is used by humans or interacts with them.
                \item \emph{Trustworthiness} aspect is a set of concerns related to the trustworthiness of CPS including security, privacy, safety, reliability, and resilience. In this paper we adopt the definition of trustworthiness from the NIST\ CPS\ Framework, where the term is taken to denote the demonstrable likelihood that the system performs according to designed behavior under any set of conditions as evidenced by its characteristics.\footnote{This is a pragmatic choice dictated by our intent to provide a formal account of the NIST CPS Framework. The debate on a universally accepted definition of trustworthiness is on-going and is beyond the scope of this paper.}  
                \item \emph{Timing} aspect: Concerns about time and frequency in CPS, including the generation and transport of time and frequency signals, time-stamping, managing latency, timing composability, etc.
                \item \emph{Data} aspect includes the concerns about data interoperability including data semantics, identify operations on data, relationships between data, and velocity of data.
                \item \emph{Communications} aspect includes the concerns about the exchange of information between components of a CPS.
                \item \emph{Boundaries} aspect is set of concerns about the interdependence among behavioral domains. Concerns related to the ability to successfully operate a CPS in multiple application area.
                \item \emph{Composition} aspect includes the concerns about the ability to compute selected properties of a component assembly from the properties of its components. Compositionality requires components that are composable: they do not change their properties in an assembly. Timing composability is particularly difficult.   
                \item \emph{Lifecycle} aspect: Concerns about the lifecycle of CPS including its components.
        \end{itemize}

The \emph{CPS Ontology} defines concepts and individuals related to concepts (with focus on \emph{Trustworthiness}) and the relationships between them (e.g.,  has-subconcern). Figure~\ref{fig:trustworthiness_2}, excluding the nodes labeled {\tt CAM}, {\tt SAM} and {\tt BAT} and links labeled ``relates'' and ``active'', shows a fragment of the CPS ontology where circle nodes represent specific concerns and grey rectangle nodes represent properties. To facilitate information sharing, the CPS Ontology leverages standards such as the Resource Description Framework (RDF\footnote{https://www.w3.org/TR/rdf-concepts/}) and the Web Ontology Language (OWL\footnote{https://www.w3.org/TR/owl-features/}) for describing the data, representing the entities and their relationships, formats for encoding the data and related metadata for sharing and fusing. An entity or relationship is defined in the ontology by an RDF-triple (\emph{subject}, \emph{predicate}, \emph{object}). Below are the main classes and relationships in the CPS ontology.


\smallskip\noindent {\em Aspects and Concerns.} The ontology defines the highest-level concept of \emph{Concern} with its refinement of \emph{Aspect}. In the concern tree in Figure~\ref{fig:cps_ontology}, the circle nodes of a concern tree represent specific concerns which are individuals of class \emph{Concern}. The root nodes of the concern tree is a particular kind of concern that is an instance of class \emph{Aspect} (subclass of \emph{Concern}). Specific concerns are represented as individuals: {\tt \small Trustworthiness} as an individual of class \emph{Aspect}, {\tt \small Security} and {\tt \small Cybersecurity} of class \emph{Concern}. Edges linking aspects and concerns are represented by the relation {\tt \small has-subconcern}. A relation {\tt \small has-subconcern} is used to associate a concern with its sub-concerns. Thus, {\tt \small Trustworthiness} aspect {\tt \small has-subconcern}  {\tt \small Security}, which in turn  {\tt \small has-subconcern} {\tt \small Cybersecurity}.

\smallskip\noindent {\em Properties.} Properties of a CPS are represented by individuals of class \emph{Property}. In the CPS Framework, a concern can be addressed by a combination of properties. An edge that links a property $p$ with an aspect or concern $c$ is represented by the relation {\tt \small addressed-by}, which says that concern $c$ is addressed by property $p$. For example in Figure~\ref{fig:trustworthiness_2} (LKAS domain), concern {\tt \small Integrity} has been addressed by some properties: {\tt \small Secure-Boot}, {\tt \small Advanced-Mode}, {\tt \small Powerful-Mode}, {\tt \small Normal-Mode} and {\tt \small Saving-Mode}.



To ease the reading, we provide a summary of the main classes and relationships in the CPS ontology 
in  Table~\ref{table:ont_class_relation}. 


%
\begin{table}[h]
\begin{center} 
\begin{tabular}{@{\extracolsep{\fill}}ll}
\topline 
\textbf{Class} & \textbf{Meaning} \\ \hline
Concern & \specialcellleft[l]{Concerns that stakeholders have w.r.t. to a system, such as \emph{security}, \emph{integrity}, etc. \\They are represented in the ontology as individuals. The link between a concern \\ and its sub-concerns is represented by the {\tt \small has-subconcern}  relation.} \\
\hline
Aspect & \specialcellleft[l]{High-level grouping of conceptually equivalent or related cross-cutting concerns \\ (i.e., \emph{human}, \emph{trustworthiness}, etc). In the ontology, \emph{Aspect} is subclass of class \\ \emph{Concern}.}  \\
\hline
Property & \specialcellleft[l]{Class of the properties relevant to a given CPS. The fact that a property addresses \\ a concern is formalized by relation {\tt \small addressed-by}.} \\
\hline
Configuration &  \specialcellleft[l]{Features of a CPS that characterize its state, e.g., if a component is on or off.\\ When property satisfaction can change at run-time, corresponding individuals \\ will be included in this class.} 
\\
\hline 
\specialcellleft[l]{Action and\\Constraint} & \specialcellleft[l]{Actions are those within the control of an agent (e.g., an operator) and those that \\ occur
spontaneously. Constraints capture dependencies among properties (e.g.,\\ mutual exclusion).
} \\ 
\botline
\textbf{Object Property} & \textbf{Meaning} \\ \hline
cpsf:hasSubCon & \specialcellleft[l]{The object property represents the {\tt \small has-subconcern} relationship between the \\ concerns.} \\ \hline 
cpsf:addrConcern & \specialcellleft[l]{The object property represents the {\tt \small addressed-by} relation between a concern \\ and a property.}
 \\
\hline 
cpsf:impactPositively & \specialcellleft[l]{The object property represents positive impact relation between a \emph{property} and a\\ \emph{concern}.}
\\
\botline
\end{tabular}
\end{center}
\caption{Main components of the CPS Ontology}
\label{table:ont_class_relation}
\vspace*{-.1in}
\end{table}

\subsection{Answer Set Programming}
Answer Set Programming (ASP)~\cite{MarekT99,Niemela99} is a declarative programming paradigm based on 
logic programming under the answer set semantics.    
A logic program $\Pi$ is a set of rules of the form:
\begin{displaymath}
c  \leftarrow a_1,\ldots,a_m, \naf b_{1},\ldots,\naf b_n
\end{displaymath}
where $c$, $a_i$'s, and $b_i$'s are literals of a propositional language\footnote{For convenience, we often use first order logic literals under the assumption that they represent all suitable ground instantiations.} and 
$\mathit{not}$ represents (default) negation. 
$c$ can be absent.  
Intuitively, a rule  states that if $a_i$'s are  believed to be true and none of the $b_i$'s is believed to be true then $c$ must be true. For a rule $r$, 
$r^\textnormal{+}$ and $r^-$, referred to as 
the \emph{positive} and \emph{negative} body, respectively, denote the sets $\{a_1,\ldots,a_m\}$ and $\{b_{1},\ldots,b_n\}$, respectively.

Let $\Pi$ be a program. An interpretation $I$ of $\Pi$ is a set of ground atoms occurring in $\Pi$.
The body of a rule $r$ is satisfied by $I$ if $r^\textnormal{+} \subseteq I$ and $r^- \cap I = \emptyset$.
A rule $r$ is satisfied by $I$ if the body of $r$ is satisfied by $I$ implies $I \models c$.  
When $c$ is absent, $r$ is a constraint and is satisfied by $I$ if its body is not satisfied by $I$. 
$I$ is a model of $\Pi$ if it satisfies all rules in $\Pi$. 

For an interpretation
$I$ and a program $\Pi$, the \emph{reduct}
of $\Pi$ w.r.t. $I$ (denoted by $\Pi^I$) is the program
obtained from $\Pi$ by deleting
{\em (i)} each rule $r$ such that $r^- \cap I \neq \emptyset$, and
{\em (ii)} all atoms of the form $\naf a$ in the bodies of the remaining rules.  
Given an interpretation $I$,
observe that the program $\Pi^I$ is a program with no occurrence of $\naf a$.
An interpretation $I$ is an \emph{answer set} \cite{GelfondL90}
of $\Pi$ if $I$ is the least model (wrt. $\subseteq$) of $\Pi^I$.   

A program $\Pi$ can have several answer sets, one answer set, or no answer set. $\Pi$ is said to be consistent if it has at least one answer set; it is inconsistent otherwise. 
Several extensions (e.g.,  \emph{choice atoms}, \emph{aggregates}, etc.) have been introduced to simplify the use of ASP. We will use and explain them when needed.
Given a program $\Pi$ and an atom $a$, we write $\Pi \models a$ to say that $a$ belongs to every answer set of $\Pi$. 
$\Pi \mid \!\!\sim  a$ to say that $a$ belongs to at least one answer set of $\Pi$.

We illustrate the concepts of answer set programming by
showing how the 3-coloring problem of a bi-directed graph $G$ can
be solved using logic programming under the answer set semantics. 
 Let the three colors be red ($r$),
blue ($b$), and green ($g$) and the vertex set of $G$ be
$\{0,1,\ldots,n\}$. Let $\Pi\left(G\right)$ be the program consisting of

\begin{itemize}
\item the set of atoms $edge\left(u,v\right)$ for every edge $\left(u,v\right)$ of $G$,

\item for each vertex  $u$ of $G$, the rule stating that $u$ must
be assigned one of the colors red, blue, or green:
$$ 1 \{color\left(u,g\right); color\left(u,r\right); color\left(u,b\right) \} 1 \leftarrow $$
This rule uses the choice atom, introduced in \cite{nie99b}, to simplify the 
use of ASP. This atom says that exactly one of the atoms 
$color\left(u,g\right)$, $color\left(u,r\right)$, and $color\left(u,b\right)$ must be true. 

\item for each edge $\left(u,v\right)$ of $G$, three rules representing the
constraint that $u$ and $v$ must have different color:
\begin{eqnarray*}
&  \leftarrow & color\left(u, r\right), color\left(v, r\right), edge\left(u,v\right) \\
&  \leftarrow & color\left(u, b\right), color\left(v, b\right), edge\left(u,v\right) \\
&  \leftarrow & color\left(u, g\right), color\left(v, g\right), edge\left(u,v\right)
\end{eqnarray*}
\end{itemize}
It can be shown that for each graph $G$, ({\it i}) $\Pi\left(G\right)$ has no answer set, i.e., is
inconsistent iff the 3-coloring problem of $G$ does not have a
solution; and ({\it ii}) if $\Pi\left(G\right)$ is consistent then each answer set of
$\Pi\left(G\right)$ corresponds to a solution of the 3-coloring problem of $G$
and vice versa.

\subsection{Action Language $\mathcal{B}$} 
\label{sub:action-language}

We review the basics of the action description language
$\mathcal{B}$ \cite{GelfondL98}. An action theory in $\mathcal{B}$ is defined over two disjoint
sets, a set of actions {\bf A} and a set of fluents {\bf F}.
A {\em fluent literal} is either a fluent $f \in \mathbf{F}$
or its negation $\neg f$.
A {\em fluent formula} is a propositional formula constructed
from fluent literals.
An action domain is a  set of laws of the following form:
\begin{align}
        \mathit{Executability \: condition{:}} & \:\quad \mathbf{executable} \:\: a  \:\: \mathbf{if} \:\: p_1,\ldots,p_n \label{exec}\\
        \mathit{Dynamic \: law{:}} & \quad a \:\: \mathbf{causes} \:\: f \:\: \mathbf{if} \:\: p_1,\ldots,p_n \label{dynamic} \\ 
        \mathit{Static \: Causal \: Law{:}} & \quad f \:\: \mathbf{if} \:\: p_1,\ldots,p_n \label{static}
\end{align}
where $f$ and $p_i$'s are fluent literals and $a$ is an
action. (\ref{exec}) encodes an executability condition of an action $a$.
Intuitively, an executability condition of the form (\ref{exec})
states that $a$ can only be executed if $p_i$'s hold. (\ref{dynamic}), referred to as a {\em dynamic causal
law}, represents the (conditional) effect of $a$. It  states that $f$ is caused to be
true after the execution of $a$ in any state of the world where
$p_1,\dots,p_n$ are true. 
When $n=0$ in \eqref{dynamic}, we often omit laws of this type from the description.  
(\ref{static}) represents a {\em static causal law}, i.e., a
relationship between fluents. It conveys that whenever the fluent literals
$p_1,\dots,p_n$ hold then so is $f$. For convenience, we sometimes denote
the set of laws of the form (\ref{static}),
(\ref{dynamic}), and (\ref{exec}) by $K$, $D_D$, and $D_E$,
respectively, for each action domain $D$.

A domain given in $\mathcal{B}$  defines a transition
function from pairs of actions and states\footnote{{\em states} are defined later} to sets of states whose
precise definition is given below. Intuitively, given an action
$a$ and a state $s$, the transition function $\Phi$ defines the
set of states $\Phi(a,s)$ that may be reached after executing the
action $a$ in state $s$. If $\Phi(a,s)$ is an empty set it means
that the execution of $a$ in $s$ results in an error.
We now formally define $\Phi$.

 Let $D$ be a domain in $\mathcal{B}$.
A set of fluent literals is said to be {\em consistent} if it
does not contain $f$ and $\neg f$ for some fluent $f$. An
{\em interpretation} $I$ of the fluents in $D$ is a maximal
consistent set of fluent literals of $D$. A fluent
$f$ is said to be true (resp. false) in $I$ iff $f \in I$ (resp.
$\neg f \in I$).
The truth value of a fluent formula in $I$ is
defined recursively over the propositional connectives in the
usual way. For example, $f \wedge g$ is true in $I$ iff $f$ is
true in $I$ and $g$ is true in $I$. We say that a formula
$\varphi$ holds in $I$ (or $I$ satisfies $\varphi$), denoted by
$I \models \varphi$, if $\varphi$ is true in $I$.

 Let $u$ be a consistent set of fluent literals and $K$ a
set of static causal laws. We say that $u$ is closed under
$K$ if for every static causal law 
$$ f \:\: \mathbf{if} \:\: p_1,\ldots,p_n$$ in $K$, if
$u \models p_1 \wedge \ldots \wedge p_n$ then $u \models f$. By
$Cl_K\left(u\right)$ we denote the least consistent set of literals from $D$
that contains $u$ and is also closed under $K$. It is worth
noting that $Cl_K\left(u\right)$ might be undefined. For instance, if $u$
contains both $f$ and $\neg f$ for some fluent $f$, then $Cl_K\left(u\right)$
cannot contain $u$ and be consistent; another example is that if
$u = \{f,g\}$ and $K$ contains 
$$ f \:\: \iif  h \quad \quad \textnormal{ and } \quad \quad \neg h \:\: \iif \:\: f, g $$ 
then $Cl_K\left(u\right)$ does not exist
because it has to contain both $h$ and $\neg h$, which means that
it is inconsistent.

 Formally, a {\em state} of $D$ is an interpretation of the
fluents in {\bf F} that is closed under the set of static causal
laws $K$ of $D$.

 An action $a$ is {\em executable} in a state $s$ if there
exists an executability proposition
$$\executable \quad a \quad \mathbf{if}  \quad f_1,\ldots,f_n $$ in $D$ such that $s \models
f_1 \wedge \ldots \wedge f_n$. Clearly, if $n=0$, then $a$ is executable in every state of $D$.
 The {\em direct effect of an action a} in a state $s$ is the
set
$$e\left(a,s\right) = \{f \mid
a \quad \causes \quad f \quad \iif  f_1,\ldots,f_n  \in D, s \models f_1 \wedge
\ldots \wedge f_n\}.
$$
 For a domain $D$,
$\Phi(a,s)$, the set of states that may be reached by executing
$a$ in $s$, is defined as follows.

\begin{enumerate}
\item If $a$ is executable in $s$, then
\[\Phi\left(a,s\right) = \{s' \; \mid \;  \; s' \mbox{ is a state and }s' =
Cl_{K}\left (e\left(a,s \right ) \cup \left (s \cap s'\right)\right)\}; \]

\item If $a$ is not executable in $s$, then $\Phi\left(a,s\right) = \emptyset$.

\end{enumerate}

Every domain $D$ in $\mathcal{B}$ has a unique
transition function $\Phi$, which we call the \emph{transition
function of $D$}. 
The transition function allows one to 
compute the set of states reached by the execution of a sequence of actions 
$\alpha = \left[a_1, \ldots ,a_n\right]$ from a state $s_0$, denoted by 
$\hat{\Phi}\left(\alpha,s_0\right)$, as follows: 

\begin{enumerate}
\item If $n=0$ then $\hat{\Phi}\left(\alpha,s_0\right) = s_0$

\item If $n> 0$ then $\hat{\Phi}\left( \alpha, s_0\right) = \cup_{u \in \Phi\left(a_1,s_0\right)} \hat{\Phi}\left( \alpha', u\right)$ where $\alpha'=[a_2,\ldots,a_n]$ and if $\hat{\Phi}\left( \alpha', u\right) = \emptyset$ 
for some $u$ then $\hat{\Phi}\left( \alpha, s_0\right) = \emptyset$.  

\end{enumerate}

\subsection{Representation and Reasoning with CPS Ontology in ASP}
\label{sub-section:bg_asp_representation}

Various researchers have explored the relationship between ASP and the Semantic Web (e.g., \cite{Eiter07,NguyenSP18A,NguyenSP18B,Nguyen2020OnRW}), in particular with the goal of leveraging existing ontologies. In these works, an ASP program is used for reasoning about classes, properties, inheritance, relations, etc. Given ASP's non-monotonic nature, it also provides sufficient flexibility for dealing in a principled way with default values, exceptions and for reasoning about the effects of actions and change. 

We use a similar approach in this paper to leverage the existing CPS Ontology for reasoning tasks related to CPS and concerns. Our approach includes the ability to query the CPS Ontology for relevant knowledge and provide it to an ASP-based reasoning component. Because the present paper is focused on the latter, for simplicity of presentation we assume that all relevant classes, instances, relations, properties of the CPS ontology are already encoded by an ASP program. We denote this program by $\Pi$($\Omega$) where $\Omega$ denotes the ontology, which is the CPS ontology in this case. 
We list the predicates that will be frequently discussed in this paper. 
\begin{itemize}
        \item  {\tt \small class(X)}:  $\mathtt{X}$ is a class;
        \item  {\tt \small subClass(X,Y)}:  $\mathtt{X}$ is a subclass of $\mathtt{Y}$;
        \item  {\tt \small aspect(I)} (resp. {\tt \small concern(I)}, {\tt \small prop(I)}, {\tt \small decomp\_func(I)}): {\tt \small I} is an individual of class aspect (resp. concern, property, decomposition function);
        \item {\tt \small subCo(I,J)}: {\tt \small J} is sub-concern of   {\tt \small I}; and
        \item {\tt \small addBy(C,P)}: concern $C$ is addressed by property $P$ (a link from a property $P$ to a concern $C$ in the ontology);
        \item {\tt \small positiveImpact(P,C)}: The satisfaction of property $P$ impacts positively on the satisfaction of concern $C$. 
        \item {\tt \small func(F,C)}: $F$ is a functional decomposition of concern $C$. 
\end{itemize}

%
\begin{lstlisting}[language=clingo,caption=$\Pi\left(\Omega\right) : $ASP program for CPS Ontology $\Omega$, label=lst:pi_Omega, mathescape=true,xleftmargin=.01\textwidth, breaklines=true][t]
class(X)      :- RDFtriple(X,"rdf:type","owl:Class").
subClass(X,Y) :- RDFtriple(X,"rdfs:subClassOf",Y), class(X), class(Y).
subClass(X,Y) :- subClass(X,Z), subClass(Z,Y).
instance(I)   :- RDFtriple(I,"rdf:type","owl:NamedIndividual").
isInstanceOf(I,X) :- instance(I), class(X), RDFtriple(I,"rdf:type",X).
isInstanceOf(I,Y) :- instance(I), class(X), class(Y), subClass(X,Y), isInstanceOf(I,X).
concern(C)     :- instance(C), isInstanceOf(C,"cpsf:Concern").
aspect(A)      :- instance(A), isInstanceOf(A,"cpsf:Aspect").
prop(P)        :- instance(P), isInstanceOf(P,"cpsf:Property").
decomp_func(F) :- instance(F), isInstanceOf(F,"cpsf:DecompFunc").
subCo(I,J) :- concern(I), concern(J), RDFtriple(I,"cpsf:hasSubCon",J).
addBy(C,P) :- prop(P), concern(C), RDFtriple(P,"cpsf:addrConcern",C).
func(F,C)  :- decomp_func(F), concern(C), RDFtriple(F,"cpsf:decompFunctionOf",C).
positiveImpact(P,C) :- concern(C), prop(P), RDFtriple(P,"cpsf:impactPositively",C).
\end{lstlisting}
Listing~\ref{lst:pi_Omega} represents the ASP program $\Pi\left(\Omega\right)$ of CPS Ontology $\Omega$. The predicate {\tt \small RDFtriple(S,P,O)} denotes the RDF triple store which has been queried and extracted from $\Omega$ by using SPARQL{\footnote{https://www.w3.org/TR/rdf-sparql-query/}}. Lines 1--2 define the {\tt \small class(X)} and {\tt \small subClass(X,Y)} based on the ontology extraction. Line 3 reasons the extension about subclass relationship. Lines 4--6 encode the definitions of {\tt \small instance(I)} and {\tt \small isInstanceOf(I,X)} with the similar method. The concern, aspect, property and decomposition function instances are defined in Lines 7--10. And, the three rules in Lines 11--14 represent the encoding of {\tt \small subCo(I,J)}, {\tt \small addBy(C,P)}, {\tt \small func(F,C)} and {\tt \small positiveImpact(P,C)} relationships respectively.

Given a collection of individuals in the CPS ontology $\Omega$, $\Pi\left(\Omega\right)$ will allow us to check  
$addBy\left(c,p\right)$, $subCo\left(i,j\right)$, $func\left(f,c\right)$, $positiveImpact\left(p,c\right)$, etc; whether a concern $c$ is addressed by a property $p$, concern $j$ is a sub-concern of concern $i$, $f$ is functional decomposition of concern $c$, the satisfaction of $p$ impacts positively on concern $c$, etc. respectively. They are written as: $\Pi\left(\Omega\right) \models addBy\left(c,p\right)$, $\Pi\left(\Omega\right) \models subCo\left(i,j\right)$, $\Pi\left(\Omega\right) \models func\left(f,c\right)$,$\Pi\left(\Omega\right) \models positiveImpact\left(p,c\right)$, etc. 


Similar rules for reasoning about the inheritance between concerns, inheritance between subconcerns and concerns, etc. are introduced whenever they are used subsequently.  
We note that the CPS framework does come with an informal semantics about when a concern is supposedly be satisfied. The work in \cite{Balduccini2018OntologyBasedRA} provides a preliminary discussion on how the satisfaction of a concern can be determined. It does not present a formal description of the CPS system as in this paper and does not address the functional decomposition issue though.

\section{CPS Theory Specification}
\label{section:cps_theory}
\subsection{Formal Definition}
\label{sub-section:cps_formal_definition}
In this section, we develop a formal definition of CPS theory and its semantics. The proposed notion of a CSP theory will allow one to specify and reason about the concerns of the CPS. Our discussion will focus on {\tt\small Trustworthiness} aspect in the CPS ontology but the proposed methodology is generic and is applicable to the full CPS ontology. To motivate the definition, we use the following example:
\begin{example} 
        [Extended from \cite{Balduccini2018OntologyBasedRA}]
        \label{example:bg_lkas_and_extend}
        Consider a lane keeping/assist system (LKAS) of an advanced car that uses a camera ({\tt CAM}) and a situational awareness module ({\tt SAM}). The SAM processes the video stream from the camera and controls the automated navigation system through a physical output. In addition, the system also has a battery ({\tt \small BAT}). 
        
        CAM and SAM may use encrypted memory ({\tt \small data\_encrypted}) and a secure boot ({\tt \small secure\_boot}). Safety mechanisms in the navigation system cause it to shut down if issues are detected in the input received from SAM. 
        The CAM and SAM can be in one of two operational modes, the basic mode ({\tt \small basic\_mode} or {\tt \small b\_mode}) and the advanced mode ({\tt \small advanced\_mode} or {\tt \small a\_mode}). The two properties address concern {\tt \small Integrity} relevant to {\tt \small operation} function. In advanced mode, the component consumes much more energy than if it were in basic mode.
        BAT serves the system energy consumption and relates with one of three properties, {\tt \small saving\_mode} ({\tt \small s\_mode}) or {\tt \small normal\_mode} ({\tt \small n\_mode}) or {{\tt \small powerful\_mode}} ({\tt \small p\_mode}). Three properties address concern {\tt \small Integrity} relevant to the {\tt \small energy} functionality.
        
        The relationship between  {\tt \small SAM}, {\tt \small CAM} and {\tt \small BAT} are: ({\bf 1}) If both {\tt \small SAM} and {\tt \small CAM} are in {\tt \small advanced\_mode}, the battery has to work in {\tt \small saving\_mode}. ({\bf 2}) if {\tt \small CAM} and {\tt \small SAM} are in {\tt \small basic\_mode}, the battery can be in {\tt \small powerful\_mode} or {\tt \small normal\_mode} and ({\bf 3}) if one of {\tt \small SAM} and {\tt \small CAM} is in {\tt \small advanced\_mode} and the other one is in {\tt \small basic\_mode}, then the battery must work in {\tt \small normal\_mode}. 
\end{example}
The relationship between the LKAS domain and the CPS ontology is shown in Figure~\ref{fig:trustworthiness_2}. Informally, the \cpsf{} defines that the concern {\tt \small Integrity} is \emph{satisfied} if {\tt \small secure\_boot} is satisfied and its two functionalities, {\tt \small operation} and {\tt \small energy},  are satisfied; the {\tt \small operation} functionality is satisfied if at least one of the properties \{{\tt \small advanced\_mode}, {\tt \small basic\_mode}\} is satisfied; and the {\tt \small energy} functionality is satisfied if there is at least one of \{{\tt \small saving\_mode}, {\tt \small normal\_mode}, {\tt \small powerful\_mode}\} properties is satisfied. Intuitively, this can be represented by the following formula:  
%
%
\begin{equation} \label{lkas_integrity} 
  \begin{array}{l}
    \left({\tt \small secure\_boot}\right) \land \left({\tt \small advanced\_mode} \lor {\tt \small basic\_mode}\right) 
    \\
    \land \left({\tt \small saving\_mode} \lor {\tt \small normal\_mode} \lor {\tt \small powerful\_mode}\right) 
  \end{array}
\end{equation}
\begin{figure*}[htbp]
        \centering
        \includegraphics[width=1.0\textwidth]{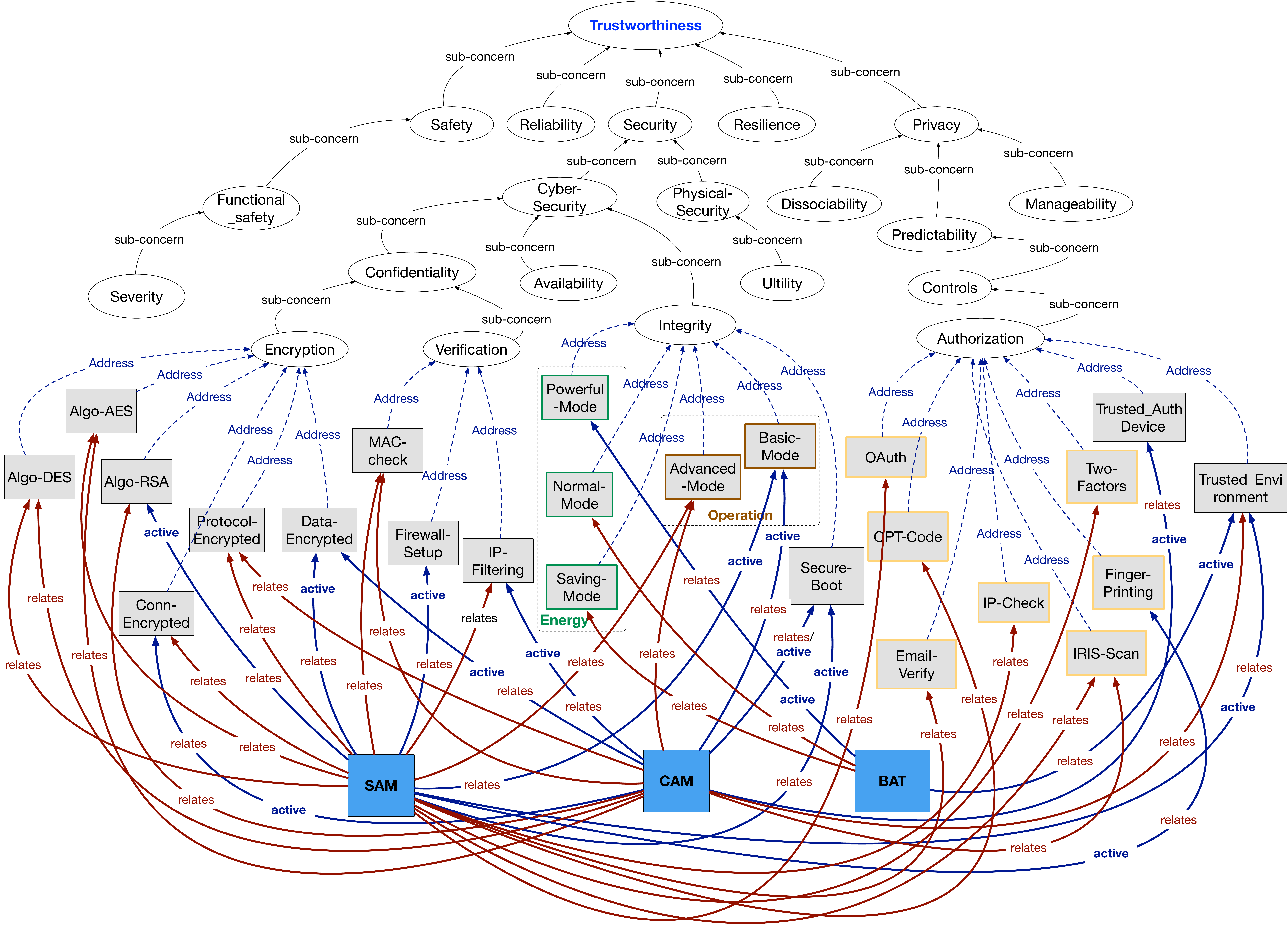}
        \caption{CPS Ontology and LKAS domain \label{fig:trustworthiness_2}}
\end{figure*}

The example shows that a CPS system is a dynamic domain and contains different components, each associated with some properties which affect the satisfaction of concerns defined in the CPS ontology. In addition, the satisfaction of concerns depends on the truth values of formulae constructed using properties and a concern might be related to a group of properties. We will write $\omega\left(c\right)$ to denote the set of properties that \emph{addresses} a concern $c$. 
We therefore define a CPS system as follows.
\begin{definition}
        [CPS System]
        \label{def:physical_CPS_system} 
        A \emph{CPS system} $\mathcal{S}$ is a tuple ($CO, A, F, R, \Gamma$) where:
        \begin{itemize} 
                \item $CO$ is a set of components; 
                \item $A$ is a set of actions that can be executed over $\mathcal{S}$; 
                \item $F$ is a finite set of fluents (or state variables) of the system;
                \item $R$ is a set of relations that maps each physical component $co \in CO$ to a set of properties $R\left(co\right)$ defined in the CPS ontology; and
                \item $\Gamma$ is a set of triples of the form  $\left(c, \mathit{fu}, \psi\right)$ where $c$ is a concern, 
                $\mathit{fu}$ is a functional decomposition of concern $c$, and $\psi$ is a formula constructed over $\omega\left(c\right)$.
        \end{itemize}
\end{definition}
In Definition~\ref{def:physical_CPS_system}, $\left(A,F\right)$ represents the dynamic domain of $\mathcal{S}$, 
$\Gamma$ represents constraints on the satisfaction of concerns in the \cpsf{} ontology in $\mathcal{S}$,
and $R$ encodes the properties of components in $\mathcal{S}$ which are related to the concerns specified in the \cpsf{}. As the truth values of these properties can be changed by actions, we assume that 
\[ 
\cup_{co \in CO} R\left(co\right) \cup \{active\left(co,p\right) \mid co\in CO, p\in R\left(co\right)\} \subseteq F. 
\] 
where $active\left(co,p\right)$ is true means that the component $co$ is currently active with property $p$. 
$\left(A,F\right)$ is an action theory  as described in Subsection~\ref{sub:action-language}. 
Note that $\left(A,F\right)$ can be non-deterministic due to the presence of statements of the form \eqref{static}. Although it is possible, this rarely happens in practical applications. We will, therefore, assume that $\left(A,F\right)$ is deterministic throughout this paper. We illustrate Definition~\ref{def:physical_CPS_system} in the following example. 

\begin{example}
        \label{example:cps_lkas}
        The CPS system in Example~\ref{example:bg_lkas_and_extend} can be described by  
        $\mathcal{S}_{lkas} = $ ($CO_{lkas}, A_{lkas}, F_{lkas}, R_{lkas}, 
        \Gamma_{lkas}$) where:
        \begin{itemize} 
                \item $CO_{lkas} = \{{\tt SAM}, {\tt CAM}, {\tt BAT}\}$.

                \item $F_{lkas}$ contains the following fluents:
                 
                \begin{itemize}
                        
                        \item $\mathtt{active\left(X,P\right)}$  denotes that component $X \in CO_{lkas}$ is working actively with property $P$, e.g., {\tt \small active(cam,basic\_mode)}, {\tt \small active(cam,data\_encrypted)}, {\tt \small active(sam,finger\_printing)} and {\tt \small active(bat,normal\_mode)} states that the camera is working in basic mode, with encrypted data, the SAM is authenticated by fingerprinting method and the battery is working in normal mode.
                        
                        \item $\mathtt{ on\left(X\right)}$ ($\mathtt{ off\left(X\right)}$) denotes that component $X$ is (isn't) ready for use.
                        
                        \item the set of properties that are related to the components ($\mathtt{P}$ denotes that the truth value of property $P$), e.g., {\tt \small basic\_mode}, {\tt \small oauth}, etc. These properties are drawn in Figure~\ref{fig:trustworthiness_2} (rectangle boxes except the three components {\tt SAM, CAM, BAT}).

                \end{itemize}

The relationship among the fluents are encoded below: 

\begin{itemize} 
\item  $active\left(BAT, saving\_mode\right) \:\:  \mathbf{if }\:\:  active\left(SAM, advanced\_mode\right), active\left(CAM, advanced\_mode\right)$ which encodes the statement 
if both {\tt \small SAM} and {\tt \small CAM} are in {\tt \small advanced\_mode}, the battery has to work in {\tt \small saving\_mode}. 

\item  $active\left(BAT, normal\_mode\right) \:\:  \mathbf{if }\:\:  active\left(SAM, advanced\_mode\right), active\left(CAM, basic\_mode\right)$ 
and \\
$active\left(BAT, normal\_mode\right) \:\:  \mathbf{if }\:\:  active\left(SAM, basic\_mode\right), active\left(CAM, advanced\_mode\right)$ 
encode the statement if one of {\tt \small SAM} and {\tt \small CAM} is in {\tt \small advanced\_mode} and the other one is in {\tt \small basic\_mode}, then the battery must work in {\tt \small normal\_mode}.

\item  $active\left(BAT, powerful\_mode\right) \lor active\left(BAT, normal\_mode\right) \:\:  \mathbf{if }\:\:  active\left(SAM, basic\_mode\right), \newline active\left(CAM, basic\_mode\right)$ which encodes the statement 
if both {\tt \small SAM} and {\tt \small CAM} are in {\tt \small basic\_mode}, the battery can be in {\tt \small powerful\_mode} or {\tt \small normal\_mode}. 


\end{itemize}

                \item $A_{lkas}$ contains the following actions:
                
                \begin{itemize} 
                
                        \item $\mathtt{switM\left(X,M\right)}$: switching the component $X$ to a mode $M$. The set of the form \eqref{exec} and \eqref{dynamic}  
                        for the action that switches the {\tt \small CAM} from {\tt basic\_mode} to {\tt advanced\_mode}  $\mathtt{ switM\left(cam, advanced\_mode\right)}$ contains the following statements: 
                        
                        \begin{itemize}
                            \item $\mathbf{executable}\:\:  \mathtt{ switM\left(cam, advanced\_mode\right)} \:\: \mathbf{if} \:\:\: \mathtt{on\left(cam\right), active
                                \left(cam, basic\_mode\right)}$ which says that the action $\mathtt{ switM\left(cam, advanced\_mode\right)}$ can only be executed if the component {\tt CAM} is on and in the {\tt basic\_mode}. 
                                
                                \item                   $\mathtt{ switM\left(cam,  advanced\_mode\right)} \:\: \mathbf{causes} \:\: \mathtt{ active\left(cam, advanced\_mode\right),}$ $\mathtt{\quad \quad \quad \quad \quad \quad \quad \quad \quad \quad \quad \quad \quad \quad \quad \quad \quad \quad \quad \quad \:\:\:\:\: \neg active\left(cam, basic\_mode\right)}$.
                                
                                This states that if we switch the component {\tt CAM} to the  {\tt advanced\_mode} then it is 
                                in the {\tt advanced\_mode} and not in the {\tt basic\_mode}. 
                                
                        \end{itemize}


                        The  statements for $\mathtt{switM\left(cam, basic\_mode\right)}$ that switches the {\tt CAM} from 
                        {\tt advanced\_mode} to {\tt basic\_mode} are similar. And the similar statements for $\mathtt{switM\left(sam, basic\_mode\right)}$
                     and $\mathtt{switM\left(sam, advanced\_mode\right)}$ which switch the component {\tt SAM} to {\tt \small basic\_mode} and {\tt \small advanced\_mode} respectively.    
                        \item There are also actions that switch other components to different modes or methods. These are: 
                        
                        \begin{itemize}
                            \item $\mathtt{switA\left(X,A\right)}$: switching between authorization methods where $X = SAM$.
                         
                            \item $\mathtt{ switV\left(X,V\right)}$: switching between verification methods where $X$ can be $SAM$ or $CAM$. 
                        
                            \item $\mathtt{ switEM\left(X,EM\right)}$: switching between encryption method  where $X$ can be $SAM$ or $CAM$.
                        
                            \item $\mathtt{ switEA\left(X,EA\right)}$: switching between encryption algorithms  where $X$ can be $SAM$ or $CAM$. 

                        \end{itemize}

                        The set of statements of the form \eqref{exec} and \eqref{dynamic} associated with these actions are similar to those associated with $\mathtt{switM\left(X,M\right)}$ and is omitted here for brevity.

                        \item $\mathtt{ tOn\left(P\right)}$ and $\mathtt{ tOff\left(P\right)}$ denote the actions of  enabling and disabling the truth value of property $P$, respectively. The sets of statements of the form \eqref{exec} and \eqref{dynamic} associated to each of these actions is similar. We list those associated with   $\mathtt{tOn\left(P\right)}$ as an example:
 
            \begin{itemize}
                \item $\mathbf{executable}\:\:  \mathtt{ tOn\left(basic\_mode\right)} \:\: \mathbf{if} \:\:\: \mathtt{\neg basic\_mode}$: this can only be executed if the system property is not in the {\tt basic\_mode}. 
                
                \item $\mathtt{ tOn\left(basic\_mode\right)} \:\: \mathbf{causes} \:\: \mathtt{ basic\_mode}$: set the system property to {\tt basic\_mode}.
                
            \end{itemize}


                        
                        
                        \item $\mathtt{ patch\left(P\right)}$ denotes action of patching some properties $P$ with available patch software. The set of statements for action $\mathtt{patch\left(P\right)}$ could be:
                        
                        {\small $\mathbf{executable}\:\:  \mathtt{ patch\left(conn\_encrypted\right)} \:\: \mathbf{if} \:\:\: \mathtt{\neg conn\_encrypted}, \mathtt{availablePatch\left(conn\_encrypted\right)}$}
                        
                        $\mathtt{ patch\left(conn\_encrypted\right)} \:\: \mathbf{causes} \:\: \mathtt{ conn\_encrypted}$
                \end{itemize}

                \item $R_{lkas} =\{${\tt CAM} $\mapsto \{{\tt \small ip\_filtering}$, ${\tt \small algo\_DES}$, ${\tt \small algo\_AES}$, ${\tt \small algo\_RSA}$, ${\tt \small data\_encrypted}$, ${\tt \small conn\_encrypted}$, ${\tt \small mac\_check}$, ${\tt \small protocol\_encrypted}$, ${\tt \small secure\_boot}$, ${\tt \small basic\_mode}$, ${\tt \small advanced\_mode}$, ${\tt \small trusted\_auth\_device}$, ${\tt \small trusted\_environment}$, ${\tt \small iris\_scan}\}$, {\tt SAM} $\mapsto \{{\tt \small data\_encrypted}$, ${\tt \small algo\_RSA}$ , ${\tt \small algo\_DES}$, ${\tt \small algo\_AES}$, ${\tt \small protocol\_encrypted}$, ${\tt \small conn\_encrypted}$, ${\tt \small firewall\_setup}$, ${\tt \small mac\_check}$, ${\tt \small ip\_filtering}$ ,${\tt \small advanced\_mode}$, ${\tt \small basic\_mode}$, ${\tt \small finger\_printing}$, ${\tt \small two\_factors}$, ${\tt \small iris\_scan}$, ${\tt \small oauth}$ , ${\tt \small opt\_code}$, ${\tt \small email\_verify}$ , ${\tt \small ip\_check}$ , ${\tt \small trusted\_environment}$ , ${\tt \small secure\_boot}\}$, {\tt BAT} $\mapsto \{{\tt powerful\_mode}$, ${\tt \small trusted\_environment}$, ${\tt \small normal\_mode}$, ${\tt \small saving\_mode}\}\}$.
                
The components and relations to the properties are illustrated by the arrow lines with ``relates'' labels in the bottom part of Figure~\ref{fig:trustworthiness_2}.

                \item $\Gamma_{lkas}$ contains the following triples (see also Figure~\ref{fig:integrity_authorization_formula}):
                
                \begin{itemize}
                        \item  ({\tt \small integrity}, {\tt \small operation}, {\tt \small advanced\_mode} $\lor$  {\tt \small basic\_mode} ) says the satisfaction of formula {\tt \small advanced\_mode} $\lor$ {\tt \small basic\_mode} addresses the concern {\tt \small integrity} in the relevant functional decomposition {\tt \small operation}.
                        
                        \item  ({\tt \small integrity}, {\tt \small energy}, {\tt \small saving\_mode} $\lor$  {\tt \small normal\_mode}  $\lor$  {\tt \small powerful\_mode}) denotes the formula {\tt \small saving\_mode} $\lor$  {\tt \small normal\_mode}  $\lor$  {\tt \small powerful\_mode} addresses the concern {\tt \small integrity} in the relevant functional decomposition {\tt \small energy}.
                        
                        \item  ({\tt \small authorization}, {\tt \small sign\_in}, {\tt \small oauth} $\land$  {\tt \small opt\_code}) denotes the satisfaction of formula {\tt \small oauth} $\land$  {\tt \small opt\_code} addresses the relevant functional decomposition {\tt \small sign\_in} of the concern {\tt \small authorization}.
                        
                        \item  ({\tt \small authorization}, {\tt \small sign\_in}, {\tt \small two\_factors} $\lor$  {\tt \small finger\_printing} $\lor$  {\tt \small iris\_scan}) denotes the formula {\tt \small two\_factors} $\lor$  {\tt \small finger\_printing} $\lor$  {\tt \small iris\_scan} addresses the concern {\tt \small authorization} in the relevant functional decomposition {\tt \small sign\_in}.
                        
                        
                        \item  ({\tt \small authorization}, {\tt \small sign\_in}, {\tt \small oauth} $\land$  {\tt \small ip\_check} $\land$  {\tt \small email\_verify}) denotes that the concern {\tt \small authorization} with the relevant functional decomposition {\tt \small sign\_in} is addressed by formula {\tt \small oauth} $\land$  {\tt \small ip\_check} $\land$  {\tt \small email\_verify}.
                        
                \end{itemize}
In addition, the functional decomposition of the {\tt \small Integrity} concern indicates that the formula $\left({\tt \small secure\_boot}\right) \land \left({\tt \small advanced\_mode} \lor {\tt \small basic\_mode}\right) \land \left({\tt \small saving\_mode} \lor {\tt \small normal\_mode} \lor {\tt \small powerful\_mode}\right)$ addresses the {\tt \small Integrity} concern.

Likewise, the formula 
\[
\begin{array}{l} 
{\tt trusted\_auth\_device} \land {\tt \small trusted\_environment} \land \\
\left({\tt  two\_factors} \lor {\tt  finger\_printing} \lor {\tt iris\_scan} \right. \lor \\
\left({\tt oauth} \land {\tt opt\_code}\right) \lor 
\left. \left( {\tt oauth} \land {\tt ip\_check} \land {\tt email\_verify}\right)\right)
\end{array} 
\]
addresses the {\tt \small Authorization} concern.
        \end{itemize}
\end{example}

\begin{figure}
        \centering
        \includegraphics[width=1.0\textwidth]{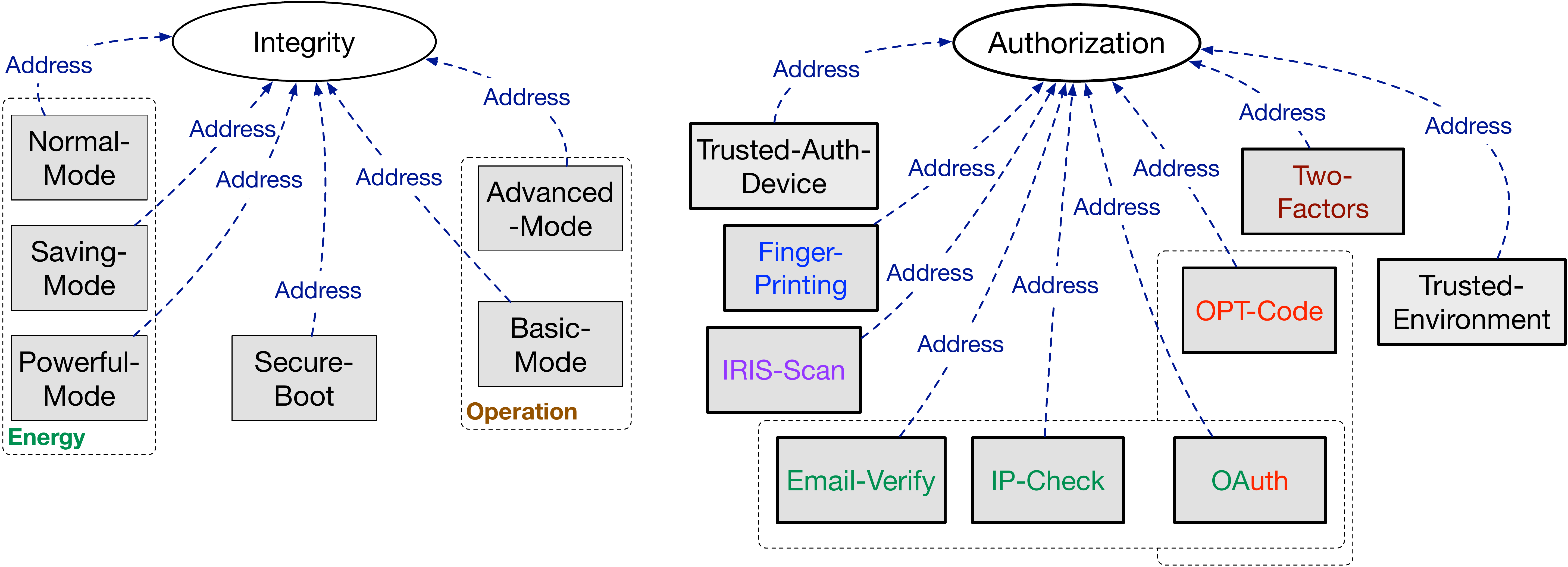}
        \caption{{\tt \small Integrity} and {\tt \small Authorization} concerns with their Functionalities and Properties \label{fig:integrity_authorization_formula}}
\end{figure}
%

%
%
Given a CPS system $\mathcal{S}$ with a set of fluents $F$, a \emph{state} $s$ of $\mathcal{S}$ is an interpretation of $F$ that satisfies the set of static causal laws of the form \eqref{static}
(Subsection~\ref{sub:action-language}). 
%
\begin{definition}
        [CPS Theory]
        \label{def:CPS_theory} 
        A \emph{CPS theory} is a pair $\left(\mathcal{S},I\right)$ where $\mathcal{S}$ is a CPS system and $I$ is a state representing the \emph{initial configuration} of $\mathcal{S}$. 
\end{definition}
%
%

\subsection{The Semantics of CPS Theories}
\label{sub-section:cps_theory_semantics}

Given $\left(\mathcal{S},I\right)$ where $\mathcal{S}=\left(CO,A,F,R,\Gamma\right)$, the action domain $\left(A,F\right)$ specifies a transition function $\Phi_\mathcal{S}$ between states (Subsection~\ref{sub:action-language}). In each state, the satisfaction of a particular concern in the CPSF is evaluated using the relationship $R$ and the components $C$. We will define this relation next. First, we note that a concern in a CPS can be related to some components in   $\mathcal{S}$, directly through the $R$ relation and the formulae in $\Gamma$ or indirectly through the inheritance in the CPS ontology. 
Observe that the development of the CPS relies on the following intuition: 

\begin{itemize}
    \item A concern might have several sub-concern; 
    \item A concern might be addressed by a set of functional decompositions which are represented by Boolean formulae. 
\end{itemize}

This leads to the following informal meaning of the notion of satisfaction of a concern in a state of the CPS:  

\begin{itemize}
   \item For each concern $c$, if $\Gamma$ does not contain any tuple of the form $\left(c,fu,\psi\right)$ 
   then $c$ is satisfied in a state $s$ when 
    every of its direct subconcerns is satisfied; for example, the 
   {\tt Trustworthiness} concern is satisfied in a state $s$ of the LKAS system if its children, 
   {\tt Safety}, {\tt Reliability}, {\tt Security}, {\tt Resilience}, and {\tt Privacy}, 
   are satisfied; and every of its properties is satisfied.

   \item For each concern $c$, if $\Gamma$ contains some tuple of the form $\left(c,fu,\psi\right)$ 
    then $c$ is satisfied when $\psi_c = \wedge_{\left(c,fu,\psi\right) \in \Gamma} \psi$ is satisfied in $s$ 
    and every property $p$ related to $c$--as specified by the CPS ontology--is satisfied in $s$; for example, the {\tt \small Integrity} concern is satisfied in the state $s$ of the LKAS system if the formula \eqref{lkas_integrity} is satisfied in $s$ where {\tt\small secure\_boot} is a property related to {\tt \small Integrity} and the other conjuncts are the two disjunctions representing the two functional decomposition of {\tt \small Integrity}.
\end{itemize}

Next, we formalize precisely the notion of satisfaction of a concern. 
Let $\Lambda(c)$ be the conjunction of 
$\wedge_{\left(c,fu,\psi\right) \in \Gamma} \psi $ and 
all properties that are related to $c$ and not appearing in 
any formula of the form $\left(c,fu,\psi\right) \in \Gamma$.
For example, in formula \eqref{lkas_integrity}, the last two conjuncts are the two functional decompositions of {\tt \small Integrity} from $\Gamma_{lkas}$ and the first conjunct is a property that does not appear in any functional decomposition of {\tt \small Integrity}. 
In the following, we denote $\langle c \rangle$ is the set of descendants of $c$ such that for each $d \in \langle c \rangle$, $d$ has no sub-concern.
\begin{definition}
   \label{def-entailment-state} 
Let $s$ be a state in $\mathcal{S}=\left(CO,A,F,R,\Gamma\right)$ and $c$ be a concern.  We say that \emph{ $c$ is satisfied in $s$}, denoted by $s \models c$, if 
\begin{itemize}
    \item $s \models \Lambda(c)$; and 
    \item every sub-concern $c'$ of $c$ is satisfied by $s$. 
\end{itemize}
\end{definition}

Having defined when a concern is satisfied in a state, we can define the notion of satisfaction of a concern after the execution of a sequence of actions as follows. Recall the transition function $\Phi_\mathcal{S}$ dictates how the system changes from one state to another state and the set of states resulting from the execution of a sequence of actions $\alpha$ from a  state can be computed by $\hat{\Phi}_\mathcal{S}$. Therefore, we can define the satisfaction of a concern $c$ after 

\begin{definition}
\label{def-entailment}
Let $\left(\mathcal{S},I\right)$ be a CPS theory, $\alpha$ a sequence of actions, and $c$ a concern in the CPS Ontology. 
$c$ is \emph{satisfied} after the execution of a sequence of actions $\alpha$ from the initial state $I$, 
denoted by $\left(\mathcal{S},I\right) \models c \after \alpha$, iff 
\begin{equation} 
     \hat{\Phi}_\mathcal{S} \lo \alpha, I \lc \ne \emptyset \wedge
    \forall u \in \hat{\Phi}_\mathcal{S} \lo \alpha, I \lc. \left[u \models c \right]
\end{equation}    
\end{definition}

In the above definition, the condition $\hat{\Phi} \lo \alpha, I \lc \ne \emptyset$ guarantees that $\alpha$ is a valid sequence of actions, i.e., its execution in $I$ does not fail. The second condition is the standard definition of logical entailment.

Definitions~\ref{def-entailment-state}-\ref{def-entailment} provide the basis for us to answer questions 
related to the satisfaction of a concern in a state or after a sequence of actions is executed, i.e., the \emph{concern satisfaction} problem. In the following, we will discuss other problems that are of importance for the design and development of CPS systems.





\subsection{Reasoning Tasks in CPS}
\label{sub:section:reasoning_tasks}
%

Knowing when a concern is (is not) satisfied is very important. We now discuss the issues related to the satisfaction of concerns in a CPS. We focus on the following problems: 


\begin{enumerate}
    \item What is the most/least trustworthy\footnote{Recall that our discussion focuses on trustworthiness but it can easily be adapted to other aspects defined in the CPS ontology.} component in a CPS?   
    \item Are there non-compliance in a given CPS? How to detect non-compliance? 
    \item What to do if an (external or internal) event occurs and leads to an undesirable situation? How to recover from such situation? 
    \item What is a best or most preferred mitigation strategy for a given situation?   
\end{enumerate}

In what follows, we provide  precise formulations of the aforementioned tasks and propose solutions for them. For simplicity of presentation, we focus on discussing these questions with respect to a given state. The answers to these questions after the execution of a sequence of actions from the initial state can be defined similarly to the definition of the satisfaction of a concern via the function $\Phi$, as in Definition~\ref{def-entailment}. Our implementation covers both situations.   

\subsubsection{Most/Least Trustworthy Components}
\label{sub-section:for_most_vulner_secure}

Given $\mathcal{S}=\left(CO,A,F,R,\Gamma\right)$ and a state $s$ in $\mathcal{S}$. A component $x \in CO$ might be related to many concerns through the properties in $R\left(x\right)$, whose truth values depend on the state $s$. Recall that for each property $p$ and component $x$, $active\left(x,p\right)$ is true in $s$ indicates that component is active with property $p$ in the state $s$; furthermore, the CPS ontology contains the specification that $p$ \emph{positively} or \emph{negatively}  impacts a concern $c$. The latter are defined by the predicates $addBy\left(c,p\right)$ and  
$positiveImpact\left(p,c\right)$ in $\Omega$ (Subsection~\ref{sub-section:bg_asp_representation}). As such, when a component is active with a property, it can positively impact a concern. 
For example, in Figure~\ref{fig:trustworthiness_2} and~\ref{fig:integrity_authorization_formula}, the property {\tt \small secure\_boot} addresses the {\tt \small Integrity} concern and is described to impact positively on the satisfaction of {\tt \small Integrity} concern by $\Omega$. In the current state, the component {\tt \small SAM} is working on property {\tt \small secure\_boot}. Assuming that concern {\tt \small Integrity} is satisfied in this state, we say that component {\tt \small SAM} \emph{directly positively affects} to the {\tt \small Integrity} concern through property {\tt \small secure\_boot}.   
We say that a  component $x$ \emph{directly impacts} a concern $c$ in state $s$ through a property $p$ if the following conditions hold:  
\begin{enumerate} 
    \item $x$ works with property $p$ in state $s$; and 
    \item $p$ addresses concern $c$ and $p$ is true in $s$.  
\end{enumerate}
If $x$ directly impacts $c$ in state $s$ through $p$ and the CPS ontology specifies that 
the satisfaction of property $p$ \emph{impacts positively} on the satisfaction of $c$ 
and $c$ is satisfied in state $s$,  
then we say that $x$ \emph{directly and positively affects} $c$.  

As the notion of concern satisfaction is propagated through the sub-concern relationship, it is natural for us to define 
that component $x$ \emph{impacts} (resp. \emph{affects positively}) concern $c$ through property $p$ in a state $s$, denoted by $impact\left(x,c,s\right)$ (resp. $pos\left(x,p,c,s\right)$), if (\emph{i}) $x$ directly addresses (resp.  direct positively affects) $c$ through a property $p$; or (\emph{ii}) there exists some sub-concern $c'$ of $c$ that is addressed (resp. positively affected) by $x$. 

In the above example (see also Figure~\ref{fig:trustworthiness_2}), the component {\tt \small SAM} \emph{directly positively affects} to the {\tt \small Integrity} concern through property {\tt \small secure\_boot} then {\tt \small SAM} also \emph{affects positively} concerns {\tt \small Cyber-Security}, {\tt \small Security} and {\tt \small Trustworthiness} in the concern tree through property {\tt \small secure\_boot}.

Given a component $x$, the ratio between the number of concerns that are positively affected by $x$ and the number of concerns that are addressed by $x$ characterizes how effectively $x$ influences the system. For this reason, we will use this number to characterize the trustworthiness of components in the system. So, we define 

\begin{equation} \label{def-twc}
tw\left(x,s\right) = \frac{\Sigma_{p \in R\left(x\right)} \mid  \{c \mid s \models c   
    \wedge positiveImpact\left (p, c \right) 
    \wedge p \in s \wedge active\left(x,p \right)\}\mid 
}{
\Sigma_{p \in R\left(x\right)} \mid \{c \mid \left(s \not\models c \vee \neg positiveImpact\left (p, c \right)\right) \wedge addBy\left (c,p \right) \wedge p \in s \wedge active\left(x,p \right) \}\mid \textnormal{+} 1
}
\end{equation} 
Assume that all concerns and properties are equally important, we could compare the trustworthiness of a component $x \in CO$ with that of a component $x' \in CO$ by comparing the ratios $tw$.

\begin{definition} 
        \label{def:select_vulnerable_sercure_comp} 
        For a CPS system $\mathcal{S} = \left(CO,A,F,R,\Gamma\right)$, $x_1, x_2 \in CO$, and state $s$ of $\mathcal{S}$,
        \begin{itemize}
                \item $x_1$ is \emph{more trustworthy} than $x_2$ in $s$, denoted by $x_1 \succ_s x_2$ (or $x_2$ is \emph{less trustworthy } than $x_1$, denoted by $x_2 \prec_s x_1$),  
                if 
                \begin{itemize}
                    \item $tw\left(x_1,s\right) > tw\left(x_2,s\right)$; or
                \item 
                $tw\left(x_1,s\right) = tw\left(x_2,s\right) = 0$ 
                and $impact\left(x_1,s\right) < impact\left(x_2,s\right)$ 
                where \\
                $impact\left(x,s\right)= \Sigma_{p \in R\left(x\right)} \mid \{c \mid \left(s \not\models c \vee \neg positiveImpact\left (p, c \right)\right) \wedge addBy\left (c,p \right) \wedge p \in s \wedge active\left(x,p \right) \}\mid$. 
                \end{itemize}
                 
                \item $x_1$ is \emph{as trustworthy as} $x_2$ in $s$, denoted by $x_1 \sim_s x_2$, 
                if 
                \begin{itemize}
                    \item                 $tw\left(x_1,s\right) = tw\left(x_2,s\right) > 0$; or  
                    \item                 $tw\left(x_1,s\right) = tw\left(x_2,s\right) = 0$ and  
                        $impact\left(x_1,s\right) = impact\left(x_2,s\right)$. 
                \end{itemize}
                
        \end{itemize}
        $x_1 \succeq_s x_2$ denotes that $x_1 \succ_s x_2$ or $x_1 \sim_s x_2$.  
        $x$ is a  most (least) trustworthy component of $\mathcal{S}$ in $s$ if 
        $x \succeq_s x'$ ($x' \succeq_s x$) for every $x' \in CO$.
\end{definition}

\begin{proposition}
        \label{prop:compare}
        Let $\mathcal{S} = \left(CO,A,F,R,\Gamma\right)$  be a CPS system and $s$ be a state in $\mathcal{S}$. The 
        relation $\succeq_s$ over the components of $\mathcal{S}$ is transitive, symmetric, and total.  
\end{proposition}
\begin{proof}
It is easy to see that for any pair of components, either $c_1 \succ_s c_2$, $c_2 \sim_s c_1$, or $c_1 \sim_s c_2$. Furthermore, $c \sim_s c$. It follows that $\succeq_s$ is therefore transitive, symmetric, and total. 
\end{proof}


\subsubsection{Non-compliance Detection in CPS}
The design of a CPS is often subject to competing constraints from various people or organizations with different focus and type of expertise. This may result in sets of constraints that are unsatisfiable, e.g., a set of concerns cannot (never) be satisfied, giving rise to a non-compliance. Example~\ref{example:conflict_type_1} shows that there exists a situation in which competing concerns cannot be satisfied at the same time. In general, the problem is formulated as follows. 

\begin{definition}
        [Lack of Compliance]
        \label{def-conflict}
Given the CPS system $\mathcal{S} = \left(CO,A,F,R,\Gamma\right)$, an integer $n$, a set of actions $SA \subseteq A$, and a set of concerns $SC$, we say that $\mathcal{S}$ is 
        \begin{enumerate}
            \item         \emph{weakly $n$-noncompliant} wrt. $\left(SA,SC\right)$ if there exists a sequence $\alpha$ of at most $n$ actions in $SA$ and an initial state $I$, such that $\left(\mathcal{S},I\right) \not\models  c \after \alpha$ for some concern $c \in SC$.
            \item         \emph{strongly $n$-noncompliant} wrt. $\left(SA,SC\right)$ if for every sequence $\alpha$ of at most $n$ actions in $SA$ and an initial state $I$,  $\left(\mathcal{S},I\right) \not\models  c \after \alpha$ for some concern $c \in SC$. 
        \end{enumerate}
\end{definition} 
Given an integer $k$, \emph{weakly $k$-noncompliant} implies that there is a potential that some concern in the set $SC$ of concerns might not be satisfied. \emph{Strongly $k$-noncompliant} indicates that there is always some concern that cannot be satisfied. Systems that are \emph{strongly $k$-noncompliant} might need to be re-designed. 

It is easy to see that, by Definition~\ref{def-entailment}, checking whether a system is \emph{weakly $k$-noncompliant} is equivalent to identifying a plan of length $k$ or less that ``makes some concern unsatisfied.'' On the other hand, checking whether a system is \emph{strongly $k$-noncompliant} is equivalent to identifying a plan  of length less than $k$ that ``satisfies all concerns''. Since we assume that the specification language for CPS is propositional and planning for bounded plans is NP-complete, we can easily derive the following results: 

\begin{proposition}
    Given $\mathcal{S}$, $\left(SA,SC\right)$, and $k$, checking whether $\mathcal{S}$ is \emph{weakly $k$-noncompliant} is NP-complete and checking whether $\mathcal{S}$ is \emph{strongly $k$-noncompliant} is co-NP-complete. 
\end{proposition}
\begin{proof}
This relies on the fact that checking whether a planning problem has a solution of length $k$ is NP-complete (e.g., the {\sc Plan-Length} problem in \cite{GhallabNT04}). 
\end{proof}

\subsubsection{Mitigation Strategies}
\label{sub-section:compute_mitigation_strategy_theory}
%

Let $\mathcal{S} = \left(CO,A,F,R,\Gamma\right)$ be a CPS system and $s$ be a state of $\mathcal{S}$. When some concerns are \emph{unsatisfied} in $s$, we need a way to \emph{mitigate} the issue. Since the execution of actions can change the satisfaction of concerns, the mitigation of an issue can be achieved by identifying a plan that suitably changes the state of properties related to the concerns. The mitigation problem in a CPS can be defined as follows:
\begin{definition}
        [Mitigation Strategy]
        \label{def-simple}
        Let  $\mathcal{S} = \left(CO,A,F,R,\Gamma\right)$ be a CPS domain and $s$ a state in $\mathcal{S}$.  
        Let $\Sigma$ be a set of concerns in $\Omega$. A \emph{mitigation strategy} addressing $\Sigma$ is a plan $\alpha$ whose execution at the initial state $s$ results in a state $s'$ such that for every $c \in \Sigma$, $c$ is satisfied in $s'$. 
          
\end{definition}

Definition~\ref{def-simple} assumes that all plans are equal. This is often not the case in a CPS system. To illustrate this issue, 

\begin{example} 
    \label{exp:mitigation}
Consider the LKAS system in Example~\ref{example:bg_lkas_and_extend}.  The initial state $I_{lkas}$ is given by: {\tt \small CAM} and {\tt \small SAM} are in {\tt \small basic\_mode} and {\tt \small secure\_boot}, {\tt \small BAT} is in {\tt \small powerful\_mode} and every properties in $I_{lkas}$ are observed to be \emph{True}. The energy consumption constraints of {\tt \small BAT} are encoded in Listing~\ref{lst:battery_constraints}. Figure~\ref{fig:integrity_likelihood_sat} shows a fragment of the CPS theory that is related to the problem described in this example.

\begin{figure}[h]
        \centering
        \includegraphics[width=200pt]{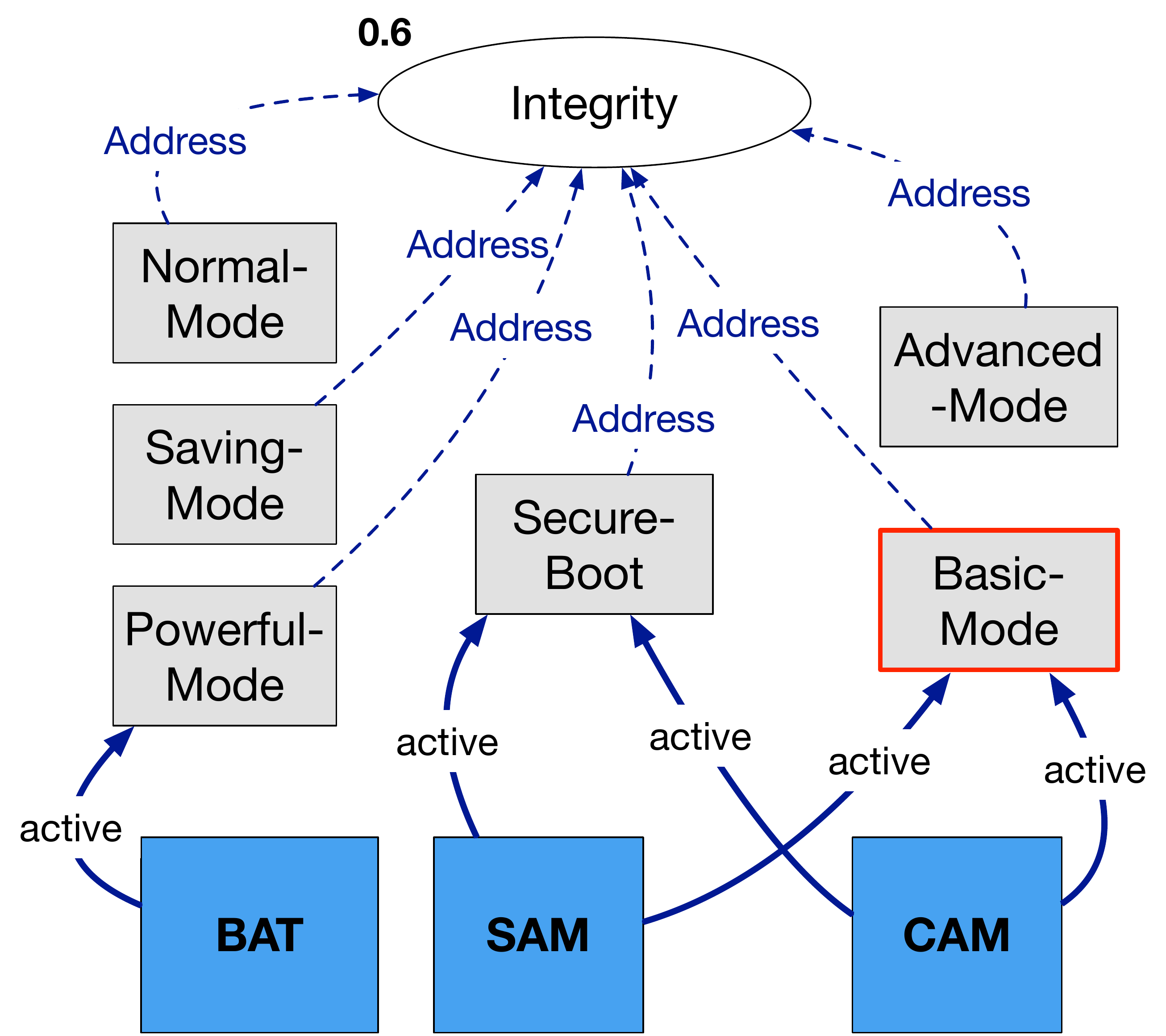}
        \caption{Current configuration of $\Delta_{lkas}$ related to {\tt \small Integrity} concern after cyber-attack\label{fig:integrity_likelihood_sat}}
\end{figure}
\begin{lstlisting}[language=clingo,caption=$\Pi^{c}_{lkas}$: Battery consumption constraints in $\Delta_{lkas}$, label=lst:battery_constraints, mathescape=true,xleftmargin=.01\textwidth, breaklines=true]
h(active(bat,saving_mode),T) :- h(active(cam,advanced_mode),T),       h(active(sam,advanced_mode),T), step(T).
1{h(active(bat,powerful_mode),T); h(active(bat,normal_mode),T)}1 :-   h(active(cam,basic_mode),T), h(active(sam,basic_mode),T), step(T).
h(active(bat,normal_mode),T) :- h(active(X,advanced_mode),T), X!=Y,   h(active(Y,basic_mode),T), step(T).
:- h(active(bat,M1),T), h(active(bat,M2),T), M1!=M2, step(T).
\end{lstlisting}

A cyber-attack occurs and the controller module is attacked, which causes {\tt \small basic\_mode} to become {\tt \small False} while {\tt \small advanced\_mode} is ({\tt \small True}). 
Given this information, we need a mitigation strategy for the set $\Sigma=\{Integrity\}$. 
The mitigation strategies (with the length is 2) can be generated as following:

\begin{list}{$\bullet$}{\itemsep=0pt \parsep=1pt \topsep=1pt \leftmargin=12pt}
        \item $\alpha_1$= $[${\tt \small tOn(basic\_mode)}$]$
        \item $\alpha_2$= $[${\tt \small switM(cam,advanced\_mode)} $,$ {\tt \small  switM(sam,advanced\_mode)}$]$
        \item $\alpha_3$= $[${\tt \small switM(sam,advanced\_mode)} $,$ {\tt \small switM(cam,advanced\_mode)}$]$
        \item $\alpha_4$= $[${\tt \small switM(sam,advanced\_mode)} $,$ {\tt \small tOn(basic\_mode)}$]$
        \item $\alpha_5$= $[${\tt \small switM(cam,advanced\_mode)} $,$ {\tt \small tOn(basic\_mode)}$]$
\end{list}
\end{example}

As shown in the example, it is desirable to identify the \emph{best} mitigation strategy. In this paper, we propose two alternatives. The first alternative relies on a notion called likelihood of satisfaction of concerns and the second alternative considers the uncertainty of actions.   

\noindent
\paragraph{\underline{Likelihood of Satisfaction (LoS) of Concerns}}

We introduce a notion called \emph{likelihood of satisfaction (LoS) of concern} and use it to distinguish mitigation strategies. Our notion relies on the positive impacts of properties on concerns within the system (Subsection~\ref{sub-section:bg_asp_representation}). For example, property {\tt \small secure\_boot} positively impacts {\tt \small Integrity} in Example~\ref{example:bg_lkas_and_extend} (denoted by {\tt \small positiveImpact(secure\_boot,integrity)}). For a concern $c$, we denote with $rel^{\textnormal{+}}\left(c\right)$ the set of all properties that positively impact a concern $c$.  
Furthermore, $rel^{\textnormal{+}}_{sat}\left(c,s\right)$ is the set of properties in $rel^\textnormal{+}\left(c\right)$ which hold in state $s$. The ratio between these two numbers can be used to characterize the \emph{positive impact degree} of concern $c$ in state $s$ as follows: 
\begin{equation}
     \label{def:pos_impacts_degree}
        deg^{\textnormal{+}}\left(c,s\right) = \left \{ 
        \begin{array}{rl}
        \dfrac{\mid rel^{\textnormal{+}}_{sat}\left(c,s\right) \mid}{\mid rel^{\textnormal{+}}\left(c\right) \mid}     &  \textnormal{if } rel^{\textnormal{+}}\left(c\right) \neq \emptyset  \\
        1     & \textnormal{otherwise}  
        \end{array}
        \right.
\end{equation}

We note that $rel^\textnormal{+}_{sat}$ and $tw$ might appear similar but they are different in the following way:   
$rel^\textnormal{+}_{sat}$ is concerned with the relationship between properties and concerns while $tw$ focuses on the relationship between components and concerns.

We define the likelihood of satisfaction of a concern as follows. 
\begin{definition}[Likelihood of Concern Satisfaction]
        \label{def:likelihood_leaf_concern}
        Given a CPS system $\mathcal{S}$, a state $s$ in $\mathcal{S}$, and a concern $c$, the likelihood of the satisfaction (LoS) of $c$ in $s$, denoted by $\varphi_{LoS}\left(c,s\right)$, is defined  by:  
        \begin{equation}
        \label{def-llh}
        \varphi_{LoS}\left(c,s\right) = \left \{
        \begin{array}{ll}
        deg^{\textnormal{+}}\left(c,s\right)*\Pi_{x \in sub\left(c\right)} \varphi_{LoS}\left(x,s\right) & \text{ if } sub\left(c\right) \neq \emptyset \\
        deg^{\textnormal{+}}\left(c,s\right) &  \text{ if } sub\left(c\right) = \emptyset
        \end{array}
        \right.
        \end{equation}
        where $sub\left(c\right)$ is the set of subconcerns of $c$. 
\end{definition}
Having defined the LoS of different concerns, we can now use this notion in comparing mitigation strategies. It is worth to mention that \cpsf{} defines nine aspect, i.e., top-level concerns, (e.g., {\tt \small trustworthiness}, {\tt \small functionality}, {\tt \small timing}, etc.). Let $TC_\Omega$ be the set of top-level concerns in the CPS ontology.  We discuss two possibilities:
\begin{itemize}
    \item \emph{Weighted LoS}:  
    Each top-level concern is associated with a number, i.e., each $c \in TC_\Omega$ is associated 
    with a weight $W_c$ (e.g., $W_{functionality}$ for {\tt \small functionality},  $W_{trustworthy}$ for  {\tt \small trustworthiness}, etc.). The
    weights represent the importance of the top-level concerns in the CPS. They can be used  
    to compute the weighted LoS of a system $\mathcal{S}$ in state $s$ 
    \begin{equation} \label{def-los-state} 
      w\left(\mathcal{S},s\right) = \Sigma_{c \in TC_\Omega} \varphi_{LoS}\left(c,s\right)*W_{c}
    \end{equation}
    This weighted LoS can be used to define a preference relation between 
    mitigation strategies such as $\beta  \prec \alpha$ ($\alpha$ is better than $\beta$)  
    iff $\max_{s' \in \Phi_\mathcal{S}(\alpha,s)} w\left(\mathcal{S},s'\right) \ge 
    \max_{s' \in \Phi_\mathcal{S}(\beta,s)} w\left(\mathcal{S},s'\right)$.

        
%
        
        \item \emph{Specified Preferences LoS}: An alternative to the weighted LoS of a system is to allow the users to specify a partial ordering over the set $TC_\Omega$ which will be used
        to define a preference relation among mitigation strategies using well-known preference aggregation strategies (e.g., lexicographic ordering). For example, if ${\tt \small Functionality >  Business}$ then a mitigation strategy $\alpha$ is better than a mitigation strategy $\beta$, 
        written as $\beta \prec \alpha$, 
        iff $\max_{s' \in \Phi_\mathcal{S}(\alpha,s)} \varphi_{LoS}\left({\tt \small Functionality},s'\right) \ge 
    \max_{s' \in \Phi_\mathcal{S}(\beta,s)} \varphi_{LoS}\left({\tt \small Business},s'\right)$ .


\end{itemize}
It is easy to see that the above preference relation $\prec$ is also transitive, symmetric, and reflexive and if some strategies exist then most preferred strategies can be computed.  

\begin{example}[Continuing from Example~\ref{exp:mitigation}]
Let us consider the strategies generated in Example~\ref{exp:mitigation}. 
All five mitigation strategies ($\alpha_1,\alpha_2,\alpha_3,\alpha_4$ and $\alpha_5$) generated in Section~\ref{sub-section:compute_mitigation_strategy} can be used to address the issue raised by the cyber-attack. 
Specifically, the fragment of final state ($G_{\alpha_i}$) relevant to {\tt \small  Integrity} concern of each plan ($\alpha_i$) is given below: 
\begin{list}{$\bullet$}{\itemsep=0pt \parsep=1pt \topsep=1pt \leftmargin=12pt}
        \item $G_{\alpha_1}$ is \{${\tt \small CAM} \mapsto {\tt \small basic\_mode}$, ${\tt \small CAM} \mapsto {\tt \small secure\_boot}$, ${\tt \small SAM} \mapsto {\tt \small basic\_mode}$, ${\tt \small SAM} \mapsto {\tt \small secure\_boot}$, ${\tt \small BAT} \mapsto {\tt \small powerful\_mode}$ \} or \{${\tt \small CAM} \mapsto {\tt \small basic\_mode}$, ${\tt \small CAM} \mapsto {\tt \small secure\_boot}$, ${\tt \small SAM} \mapsto {\tt \small basic\_mode}$, ${\tt \small SAM} \mapsto {\tt \small secure\_boot}$, ${\tt \small BAT} \mapsto {\tt \small normal\_mode}$ \}. 
        
        In which, we define $G^1_{\alpha_1}$ is \{${\tt \small CAM} \mapsto {\tt \small basic\_mode}$, ${\tt \small CAM} \mapsto {\tt \small secure\_boot}$, ${\tt \small SAM} \mapsto {\tt \small basic\_mode}$, ${\tt \small SAM} \mapsto {\tt \small secure\_boot}$, ${\tt \small BAT} \mapsto {\tt \small powerful\_mode}$ \}, and $G^2_{\alpha_1}$ is \{${\tt \small CAM} \mapsto {\tt \small basic\_mode}$, ${\tt \small CAM} \mapsto {\tt \small secure\_boot}$, ${\tt \small SAM} \mapsto {\tt \small basic\_mode}$, ${\tt \small SAM} \mapsto {\tt \small secure\_boot}$, ${\tt \small BAT} \mapsto {\tt \small normal\_mode}$ \}.
        
        \item $G_{\alpha_2}$ and $G_{\alpha_3}$: \{${\tt \small CAM} \mapsto {\tt \small advanced\_mode}$, ${\tt \small CAM} \mapsto {\tt \small secure\_boot}$, ${\tt \small SAM} \mapsto {\tt \small advanced\_mode}$, ${\tt \small SAM} \mapsto {\tt \small secure\_boot}$, ${\tt \small BAT} \mapsto {\tt \small saving\_mode}$\}
        
        \item $G_{\alpha_4}$ is \{${\tt \small CAM} \mapsto {\tt \small basic\_mode}$, ${\tt \small CAM} \mapsto {\tt \small secure\_boot}$, ${\tt \small SAM} \mapsto {\tt \small advanced\_mode}$, ${\tt \small SAM} \mapsto {\tt \small secure\_boot}$, ${\tt \small BAT} \mapsto {\tt \small normal\_mode}$\}
        
        \item $G_{\alpha_5}$ is \{${\tt \small CAM} \mapsto {\tt \small advanced\_mode}$, ${\tt \small CAM} \mapsto {\tt \small secure\_boot}$, ${\tt \small SAM} \mapsto {\tt \small basic\_mode}$, ${\tt \small SAM} \mapsto {\tt \small secure\_boot}$, ${\tt \small BAT} \mapsto {\tt \small normal\_mode}$\}
\end{list}

In each considered state, the statement $X \mapsto P$ denotes that component X is working with property $P$. For example, ${\tt \small BAT} \mapsto {\tt \small saving\_mode}$ says that the battery is working in saving mode.

Considering the five final configurations of different mitigation strategies in the example above, we have: \\
\hspace*{1cm}  $deg^{\textnormal{+}}\left(Integrity,G^1_{\alpha_1}\right)$ = 0.6, \quad $\varphi_{LoS}\left(Integrity,G^1_{\alpha_1}\right)$ = 0.6 ;\\  
\hspace*{1cm}  $deg^{\textnormal{+}}\left(Integrity,G^2_{\alpha_1}\right)$ = 0.4, \quad $\varphi_{LoS}\left(Integrity,G^2_{\alpha_1}\right)$ = 0.4 ;\\
\hspace*{1cm}  $deg^{\textnormal{+}}\left(Integrity,G_{\alpha_2}\right)$ = 0.8,  \quad $\varphi_{LoS}\left(Integrity,G_{\alpha_2}\right)$ = 0.8;\\ 
\hspace*{1cm}  $deg^{\textnormal{+}}\left(Integrity,G_{\alpha_3}\right)$ = 0.8, \quad $\varphi_{LoS}\left(Integrity,G_{\alpha_3}\right)$ = 0.8; \\
\hspace*{1cm}  $deg^{\textnormal{+}}\left(Integrity,G_{\alpha_4}\right)$ = 0.6, \quad $\varphi_{LoS}\left(Integrity,G_{\alpha_4}\right)$ = 0.6 and\\
\hspace*{1cm}  $deg^{\textnormal{+}}\left(Integrity,G_{\alpha_5}\right)$ = 0.6, \quad $\varphi_{LoS}\left(Integrity,G_{\alpha_5}\right)$ = 0.6\\
We also have that $deg^{\textnormal{+}}\left(availability,\_\right)$ = 1,  $deg^{\textnormal{+}}\left(security,\_\right)$ = 1, $deg^{\textnormal{+}}\left(trustworthiness,\_\right)$ = 1, etc. In addition, we also have the LoS values of {\tt \small trustworthiness} aspect in the five different final configurations as following:\\
\hspace*{1cm}  $\varphi_{LoS}\left(Trustworthiness,G^1_{\alpha_1}\right)$ = 0.0497, \\  
\hspace*{1cm}  $\varphi_{LoS}\left(Trustworthiness,G^2_{\alpha_1}\right)$ = 0.0331, \\ 
\hspace*{1cm}  $\varphi_{LoS}\left(Trustworthiness,G_{\alpha_2}\right)$ = 0.0662,  \\ 
\hspace*{1cm}  $\varphi_{LoS}\left(Trustworthiness,G_{\alpha_3}\right)$ = 0.0662,  \\
\hspace*{1cm}  $\varphi_{LoS}\left(Trustworthiness,G_{\alpha_4}\right)$ = 0.0497, and \\
\hspace*{1cm}  $\varphi_{LoS}\left(Trustworthiness,G_{\alpha_5}\right)$ = 0.0497.\\

Figure~\ref{fig:trust_tree_LoSC} shows the {\tt \small trustworthiness} tree for the final configurations of mitigation strategies $\alpha_2$ and $\alpha_3$ ($G_{\alpha_2}$ and $G_{\alpha_3}$), where LoS values are computed and displayed as a number at the top-left of each concern. In all 5 possible strategies, mitigation strategies $\alpha_2$ and $\alpha_3$ are also the best mitigation strategies which are especially relevant to the {\tt \small trustworthiness} attribute, where the LoS of {\tt \small trustworthiness} aspect in final state ($G_{\alpha_2}$ and $G_{\alpha_3}$) is maximum. In this figure, the LoS of {\tt \small trustworthiness} (root concern) is 0.0662 ({\tt \small llh\_sat(trustworthiness)=0.0662}). By applying a similar methodology for all remaining aspects (i.e., {\tt  business}, {\tt functional}, {\tt  timing} etc.), we can calculate LoS values for all nine aspects in CPS Ontology.
\end{example}

\begin{figure}
        \centering
        \includegraphics[width=1.0\textwidth]{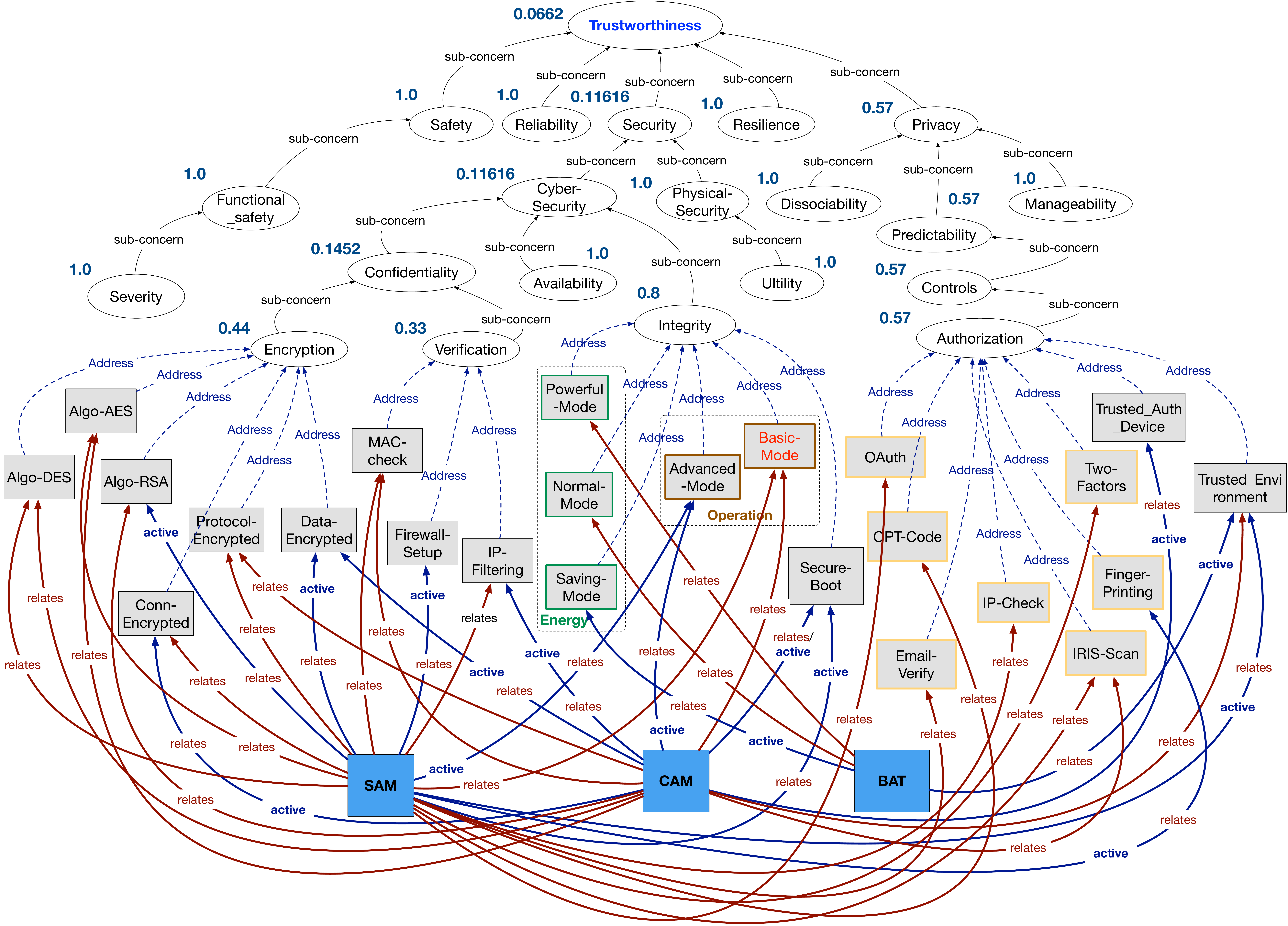}
        \caption{Trustworthiness concern tree with LoS of concerns computation \label{fig:trust_tree_LoSC}}
\end{figure}

\paragraph{\underline{Mitigation Strategy with The Best Chance to Succeed}}
\label{sub-section:for_probability_success_theo}
Preferred mitigation strategies computed using LoS of concern satisfaction assume that actions always succeeded. In practice, actions might not always succeed. In this case, it is preferable to identify strategies with the best chance of success. 
Assume that each action $a$ is associated with a set of statements of the form: 
\begin{equation} \label{prob-defined}
\quad {a \quad \prob{} \quad v \quad \mathbf{if} \quad X}
\end{equation}
where $v \in \left[0,1\right]$ and $X$ is a consistent set of literals in $\mathcal{S}$. This statement says that if each $l \in X$ is true in a state $s$ and $a$ is executable in $s$ then $v$ is the probability of $a$'s execution in $s$ succeeds. We assume that if $a$ occurs in two statements 
``$a \:\: \prob{} \:\: v_1 \:\:  \mathbf{if} \:\:  X_1$'' and 
``$a \:\: \prob{} \:\: v_2 \:\:  \mathbf{if} \:\:  X_2$'' with $X_1 \ne X_2$ then $v_1 = v_2$ or 
there exists $p \in F$ such that $\{p, \neg p\} \subseteq X_1 \cup X_2$. Furthermore, for a state $s$ in which no statement associated with some action $a$ is applicable, we assume that $a$ succeeds with probability 1 in $s$ if it is executable in $s$. It is easy to see that this set of statements defines a mapping $pr: A \times States \rightarrow \left[0,1\right]$ where $States$ denotes the set of all states of $\mathcal{S}$ and $pr\left(a,s\right)$ represents the probability that the execution of $a$ in $s$ succeeds. Thus, the execution of a sequence of actions (or a strategy) $\alpha = \left[a_0,\ldots,a_{n-1}\right]$ in a state $s$ succeeds with the probability $\Pi_{i=0}^{n-1} pr\left(a_i,s_i\right)$ where $s_0=s$, and  for $i>0$, $s_i$ is the result of the execution of $a_{i-1}$ in $s_{i-1}$. This can be used to define a preference relation between strategies similar to the use of LoS of concern satisfaction, i.e., we prefer strategies whose probability of success is maximal. We omit the formal definition here for brevity.

It is worth mentioning that the specification by statements of the form \eqref{prob-defined} is at the action level. It is assumed that if action $a$ succeeds with a probability $v$, it means that all of its potential effects will be achieved with the probability $v$. In some applications, it might be more proper to consider a finer level of probabilistic specification of effects such as if action $a$ succeeds then with a probability $p_i$, $e_i$ will be true, for $i=1,\ldots,k$. To work with this type of applications, a probabilistic action language such as the one proposed in \cite{BaralTT02} or a specification using Markov decision process could be used. We will leave the discussion related to this type of applications for the future.     



\section{An ASP-Based Implementation for Reasoning Tasks in CPS Theories}
\label{section:implementation}

This section develops an ASP encoding given a CPS theory, building on the work on planning in ASP and on formalizing CPS (e.g., ~\cite{GelfondL93,Balduccini2018OntologyBasedRA}). 
The code is available at \url{https://github.com/thanhnh-infinity/Research_CPS}.
We start with the encoding of the theory (Subsection~\ref{sub-section:asp_encoding_cps_theory}).  
Afterwards, we develop, for each reasoning task, an ASP module 
(Subsections~\ref{sub-section:concern_satisfaction}--\ref{sub-section:for_probability_success_imp}) which, when added to the encoding of the domain, will compute the answers for the task.

Throughout this section, we assume that $\left(\mathcal{S},I\right)$ where $\mathcal{S}=\left(CO,A,F,R,\Gamma\right)$ is a CPS. The encoding of $\left(\mathcal{S},I\right)$
in ASP will be denoted with $\Pi\left(\mathcal{S}\right)^n$, where $n$ is a non-negative integer representing the horizon of the system that we are interested in. We note that the encoding of the CPS ontology (Subsection~\ref{sub-section:bg_CPS_ontology} and~\ref{sub-section:bg_asp_representation}), $\Pi(\Omega)$, will be automatically added to any program developed in this section. For this reason, whenever we write 
$\Pi\left(\mathcal{S}\right)^n$ we mean 
$\Pi\left(\mathcal{S}\right)^n \cup \Pi(\Omega)$.

\subsection{ASP Encoding of a CPS Theory}
\label{sub-section:asp_encoding_cps_theory}

The encoding of a CPS theory contains two parts, one encodes the domain and another the initial state. 
We first discuss the encoding of the domain. 

\subsubsection{Encoding of the Domain $\mathcal{S}$}
\label{sub-sub-section:asp_encoding_domain}
\hfill\\
$\Pi\left(\mathcal{S}\right)^n$ contains the following rules\footnote{We follow the convention in logic programming and use strings starting with lower/uppercase letter to denote constants/variables. In addition, this program can be generated automatically given that $\mathcal{S}$ is specified in the syntax given in Section~\ref{section:cps_theory}.}.
\begin{itemize} 
        \item \emph{The set of rules declaring the time steps:}
        for each $0 \le t \le n$, an atom $step\left(t\right)$, i.e., the rule $step\left(t\right) \leftarrow $. 
        \item \emph{The set of rules encoding the components:} for each $co \in CO$, an atom $comp\left(co\right)$.
        \item \emph{The set of rules encoding actions:} for each $a \in A$, an atom $action\left(a\right)$.
        \item \emph{The set of rules encoding fluents:} for each $f \in F$, an atom $fluent\left(f\right)$.
        \item \emph{The set of rules encoding relations:}
        for each $co \in CO$ and $p \in R\left(co\right)$, an atom $relation\left(co,p\right)$.
        \item \emph{The set of rules encoding functional dependencies:}
        for each $\left(c, fu, \varphi\right) \in \Gamma$, 
        an atom $formula\left(id_\varphi\right)$, an atom $addFun\left(c,fu, id_\varphi\right)$, and a set of atoms encoding $\varphi$, 
        where $id_\varphi$ is  a unique identifier associated to $\varphi$ and $c$ is a concern.
        \item \emph{The rules for reasoning about actions and changes (see, e.g., \cite{SonBTM06})}:
        \begin{itemize} 
                \item For each executability condition of the form~\eqref{exec} the rule: \\
                $\mathtt{ exec\left(a,T\right)\:{:}{-}\:step\left(T\right),
                        \:h^*\left(p_1,T\right),\ldots,h^*\left(p_n,T\right).}$
                \item For each dynamic causal law of the form~\eqref{dynamic}: \\
                $\mathtt{ h^*\left(f,T \textnormal{\textnormal{+}} 1\right)\:{:}{-}\:step\left(T\right),\:occurs\left(a,T\right),\:}$ 
                $\mathtt{ h^*\left(p_1,T\right),\ldots,h^*\left(p_n,T\right).}$
                \item For each state constraint of the form~\eqref{static}: \\
                $\mathtt{ h^*\left(f,T\right)\:{:}{-}\:step\left(T\right),\:h^*\left(p_1,T\right),\ldots,h^*\left(p_n,T\right).}$
                \item The rules encoding the inertia axiom: \\
                $\mathtt{ h\left(f,T\textnormal{+}1\right)\:{:}{-}\:step\left(T\right), h\left(f,T\right), \naf\: \neg h\left(f,T\textnormal{+}1\right).}$ \\
                $\mathtt{ \neg h\left(f,T\textnormal{+}1\right)\:{:}{-}\:step\left(T\right), \neg h\left(f,T\right), \naf\: h\left(f,T\textnormal{+}1\right).}$ 
        \end{itemize}
        where $\mathtt{ h^*\left(x,T\right)}$ stands for $\mathtt{ h\left(x,T\right)}$ if $x  \in F$ is a fluent 
        and $\mathtt{ \neg h\left(y,T\right)}$ if $x = \neg y$ and $y \in F$. 
\end{itemize} 
We illustrate the ASP encoding of a CPS by presenting the encoding of the LKAS theory in Example~\ref{example:bg_lkas_and_extend}. Listing~\ref{lst:pi_S} shows the encoding of components, actions, and relations of $\mathcal{S}_{lkas}$ without the encoding of the initial state. 
Listing~\ref{lst:lkas_authorization_property} shows the ASP encoding for $\Gamma_{lkas}$ (see  Figure~\ref{fig:integrity_authorization_formula}). 
%
\begin{lstlisting}[language=clingo,caption=Example program $\Pi\left(\mathcal{S}_{lkas}\right)^n$ for LKAS, label=lst:pi_S, mathescape=true,xleftmargin=.01\textwidth, breaklines=true]
comp(sam). comp(cam). comp(bat). 
relation(cam,algo_AES). relation(cam,algo_RSA). 
relation(cam,algo_DES). relation(cam,ip_filtering).
relation(cam,conn_encrypted). relation(cam,data_encrypted). 
relation(cam,protocol_encrypted). relation(cam,mac_check). 
relation(cam,secure_boot). relation(cam,iris_scan).
relation(cam,advanced_mode). relation(cam,basic_mode). 
relation(cam,trusted_auth_device). relation(cam,trusted_environment).
relation(sam,algo_AES). relation(sam,algo_RSA). 
relation(sam,algo_DES). relation(sam,mac_check).
relation(sam,conn_encrypted). relation(sam,data_encrypted). 
relation(sam,ip_filtering). relation(sam,secure_boot).
relation(sam,protocol_encrypted). relation(sam,firewall_setup).
relation(sam,advanced_mode). relation(sam,basic_mode). 
relation(sam,finger_printing). relation(sam,two_factors). 
relation(sam,iris_scan). relation(sam,oauth).
relation(sam,opt_code). relation(sam,email_verify). 
relation(sam,ip_check). relation(sam,trusted_environment).
relation(bat,powerful_mode). relation(bat,normal_mode). 
relation(bat,saving_mode). relation(bat,trusted_environment).
...
action(tOn(X))  :- prop(X). action(tOff(X)) :- prop(X). 
exec(tOn(X),T)  :- $\neg$h(X,T), prop(X), step(T).
exec(tOff(X),T) :- h(X,T), prop(X), step(T).
h(X,T+1)  :- occurs(tOn(X),T), step(T).
$\neg$h(X,T+1) :- occurs(tOff(X),T), step(T).
action(patch(X)):- prop(X).
exec(patch(X),T):- prop(X), availablePatch(X), $\neg$h(X,T), step(T).
h(X,T+1)        :- occurs(patch(X),T), step(T).
...
action(switM(cam,basic_mode)). action(switM(cam,advanced_mode)).
action(switM(sam,basic_mode)). action(switM(sam,advanced_mode)).
action(switM(bat,saving_mode)). action(switM(bat,normal_mode)). 
action(switM(bat,powerful_mode)).
exec(switM(X,basic_mode),T) :- relation(X,basic_mode), 
    not h(active(X,basic_mode),T), comp(X), h(basic_mode,T), step(T).
h(active(X,basic_mode),T+1) :- occurs(switM(X,basic_mode),T), step(T).
$\neg$h(active(X,advanced_mode),T+1) :- occurs(switM(X,basic_mode),T), 
    h(active(X,advanced_mode),T), step(T).
exec(switM(X,advanced_mode),T) :-  comp(X), relation(X,advanced_mode), 
    not h(active(X,advanced_mode),T), h(advanced_mode,T), step(T).
h(active(X,advanced_mode),T+1) :- occurs(switM(X,advanced_mode),T), step(T).
$\neg$h(active(X,basic_mode),T+1):- step(T), h(active(X,basic_mode),T), occurs(switM(X,advanced_mode),T).
...
\end{lstlisting}
%
In Listing~\ref{lst:pi_S}, Line~1 encodes the components; Lines 2--20 encode the relations; 
Lines 22--29 encode the actions {\tt \small tOn} and {\tt \small tOff}. The remaining lines of code encode other actions in similar fashion. 

Each \emph{formula}  $\varphi$ related to a concern $c$ is associated with a unique identifier $\varphi^I$ and is  
converted into a CNF $\varphi_1 \wedge \ldots \wedge \varphi_k$, 
each $\varphi_i$ will be associated with a unique identifier $\varphi_i^I$. 
The set of identifiers are declared using the predicate {\tt \small formula/1}. 
It will be declared as {\tt \small disjunction} or {\tt \small conjunction}. 
Furthermore, set notation is used to encode a disjunction or conjunction, i.e.,   
the predicate {\tt \small member(X,G)} states that the formulae $X$ is a member of a disjunction or a conjunction $G$.
The predicate {\tt \small func(F,C)} states that $F$ is the functional decomposition of concern $C$.

\begin{lstlisting}[language=clingo,caption=A part of ASP program $\Pi\left(\mathcal{S}_{lkas}\right)^n$ encoding $\Gamma_{lkas}$ for {\tt Integrity} and {\tt Authorization} concerns, label=lst:lkas_authorization_property, mathescape=true,xleftmargin=.01\textwidth, breaklines=true]
formula(0..3).
...
concern(integrity). 
conjunction(0). addConcern(integrity,0).  
member(secure_boot,0). member(energy_func,0). 
member(operation_func,0). 
func(operation_func,integrity). func(energy_func,integrity).
disjunction(operation_func). formula(operation_func).
member(advanced_mode,operation_func). 
member(basic_mode,operation_func).
disjunction(energy_func). formula(energy_func).
member(powerful_mode,energy_func). member(normal_mode,energy_func).
member(saving_mode,energy_func).
...
concern(authorization). 
conjunction(1). addConcern(authorization,1).
member(trusted_auth_device,1). 
member(trusted_environment,1). 
member(sign_in_func,1).
func(sign_in_func,authorization). 
disjunction(sign_in_func). 
formula(sign_in_func).
member(finger_printing,sign_in_func). 
member(iris_scan,sign_in_func).
member(two_factors,sign_in_func).
member(2,sign_in_func). member(3,sign_in_func).
conjunction(2). 
member(oauth,2). member(opt_code,2).
conjunction(3). 
member(oauth,3). member(ip_check,3). member(email_verify,3).
...
\end{lstlisting}

%
%

In Listing~\ref{lst:lkas_authorization_property}, the first line uses a special syntax, a short hand, declaring four atoms 
{\tt \small formula(0)},$\ldots$,{\tt \small formula(3)}. The declaration and encoding of the {\tt \small Integrity} concern and its related formulas, properties and decomposition functions are presented in Lines 3--13. In which, line 3 declares the concern {\tt \small Integrity}. Lines 4--6 encode the conjunctive formula ({\tt \small conjunction(0)}) that addresses the {\tt \small Integrity} concern and its membership (e.g., the property {\tt \small secure\_boot} and the two decomposition functions of the {\tt \small Integrity} concern). Line 7 specifies the two functional dependencies of the {\tt \small Integrity} concern which are 
{\tt \small operation\_func} and {\tt \small energy\_func}. Lines 8--13 specify how the formulae address the functional decompositions. Lines 8--10 declare the disjunctive formula {\tt \small operation\_func} and define the membership between properties and this formula (e.g., 
{\tt \small member(advanced\_mode,operation\_func)}, {\tt \small member(basic\_mode,operation\_func)} says that {\tt \small advanced\_mode} and {\tt \small basic\_mode} are elements of the disjunction {\tt \small operation\_func}). Similar encoding is applied for disjunctive formulae {\tt \small energy\_func} in Lines 11--13. Lines 15--30 encode information related to the {\tt \small Authorization} concern.

\subsubsection{Encoding of the Initial State}
\label{sub-sub-section:asp_encoding_initial_state}
\hfill\\
The encoding of the initial state $I$ of a CPS theory $\left(\mathcal{S},I\right)$, denoted by $\Pi\left(I\right)$, contains, for each fluent $f$, 
$h\left(f, 0\right)$ if $f$ is true in $I$ 
or $\neg h\left(f, 0\right)$ if $f$ is false in $I$.
Listing~\ref{lst:init_config_lkas} shows a snippet of the initial state of $\mathcal{S}_{lkas}$  
with Lines 1--7 specifying the true/false properties and Lines 9--17 the specific information about which components operate in which properties in LKAS in the initial state.

\begin{lstlisting}[language=clingo,caption=An example for a part of initial configuration of $\Pi\left(I_{lkas}\right)$, label=lst:init_config_lkas, mathescape=true,xleftmargin=.01\textwidth, breaklines=true]
h(finger_printing,0). h(oauth,0). h(ip_check,0). 
h(two_factors,0). h(opt_code,0).
h(trusted_auth_device,0). h(trusted_environment,0). h(secure_boot,0). 
h(powerful_mode,0). h(saving_mode,0). h(normal_mode,0).
h(basic_mode,0). h(advanced_mode,0).
...
$\neg$h(iris_scan,0). $\neg$h(email_verify,0). $\neg$h(firewall_setup,0).
...
h(active(sam,secure_boot),0). h(active(sam,algo_RSA),0). 
h(active(sam,basic_mode),0). h(active(sam,data_encrypted),0). 
h(active(sam,firewall_setup),0). h(active(sam,finger_printing),0).
h(active(sam,trusted_environment),0). 
h(active(cam,ip_filtering),0). h(active(cam,data_encrypted),0).
h(active(cam,conn_encrypted),0). h(active(cam,secure_boot),0).
h(active(cam,trusted_auth_device),0). h(active(cam,basic_mode),0).
h(active(bat,powerful_mode),0). h(active(bat,trusted_environment),0).
...
\end{lstlisting}


The following property (see, \cite{SonBTM06}) will be important for our discussion. It shows that 
$\Pi\left(\mathcal{S}\right)^n$ correctly computes the function $\Phi_\mathcal{S}$. 

\begin{proposition}
    Let $s$ be a state in $\mathcal{S}$. Let $\Pi = \Pi\left(\mathcal{S}\right)^1 \cup \{h^*(f, 0) \mid f \in s\}$. 
    Assume that $a$ is an action that is executable in $s$. 
    Then, $s' \in \Phi_\mathcal{S}(a, s)$ iff 
    there exists an answer set $S$ of $\Pi \cup \{occurs\left(a,0\right)\}$ 
    such that $\{h^*(f, 1) \mid f \in s'\} \subseteq A$.
\end{proposition}

It is worth mentioning that $\Pi\left(\mathcal{S}\right)^n$ 
allows us to reason about effects of actions in the following sense: 
assume that $\left[a_0,\ldots,a_{n-1}\right]$ is a sequence of actions, then 
$\Pi\left(\mathcal{S}\right)^n  \cup \{occurs\left(a_i, i\right) \mid i=0, \ldots, n-1\}$ has an answer set $S$ if and only if 
(\emph{i}) $a_0$ is executable in the state $I$; 
(\emph{ii}) for each $i > 0$, $a_{i}$ is executable after the execution of the sequence $\left[a_0,\ldots,a_{i-1}\right]$; 
(\emph{iii}) for each $i$, the set 
$\{f \mid f\in F, h\left(f,i\right) \in S\} \cup \{\neg f  \mid f\in F, \neg h\left(f,i\right) \in S\}$
is a state of $\mathcal{S}$.
\subsection{Computing Satisfaction of Concerns}
\label{sub-section:concern_satisfaction}
We will next present a set of ASP rules for reasoning about the satisfaction of concerns as specified in Definitions~\ref{def-entailment-state}--\ref{def-entailment}. Since a concern is satisfied if \emph{all} of its functional decompositions and properties are satisfied, we define rules for computing the predicate
     $\mathtt{ h\left(sat\left(C\right),T\right)}$ which states that concern $C$ is satisfied at the step $\mathtt{T}$.
The rules are given in Listing~\ref{lst:satisfaction_reasoning}.

\begin{lstlisting}[language=clingo,caption= $\Pi_{sat}$ : Concern Satisfaction Reasoning in $\Omega$,label=lst:satisfaction_reasoning, mathescape=true,xleftmargin=.01\textwidth, breaklines=true]
formula($\neg$G) :- formula(G). 
prop($\neg$G)     :- prop(G). 
h($\neg$F,T):- step(T), 1{formula(F);prop(F)}, $\neg$h(F,T).
h(F,T) :- step(T), formula(F), disjunction(F), member(G,F), h(G,T). 
$\neg$h(F,T):- step(T), formula(F), disjunction(F), not h(F,T). 
$\neg$h(F,T):- step(T), 1{formula(G);prop(G)}, formula(F), conjunction(F), member(G,F), not h(G,T). 
h(F,T) :- step(T), formula(F), conjunction(F), not $\neg$h(F,T).  
$\neg$h(sat(C),T) :- concern(C), addConcern(C,F), not h(F,T), step(T).
$\neg$h(sat(X),T) :- subCo(X,Y), not h(sat(Y),T), concern(X), concern(Y), step(T).
$\neg$h(sat(X),T) :- subCo(X,Y), $\neg$h(sat(Y),T), concern(X), concern(Y),      step(T).
h(sat(C),T) :- not $\neg$h(sat(C),T), concern(C), step(T).
\end{lstlisting}

The first two lines declare that the negation of a formula or a property is also a formula and thus can be a member of a disjunction or conjunction. 
The rule on Line 3 says that $\mathtt{h\left(\neg F,T\right)}$ is true if 
the negation of {\tt \small F} is true. This rule uses a special syntax {\tt \small 1\{formula(F);prop(F)\}} which says that there exists at least one {\tt \small F} is both a formula and a property. 
The rule on Line 4 states that $\mathtt{h\left(F,T\right)}$ is true if 
{\tt \small F} is a disjunction and one of its disjuncts is true. 
The next rule (Line 5) states that $\mathtt{\neg h\left(F,T\right)}$ for a disjunction {\tt \small F} is true if it cannot be proven that {\tt F} is true.  This rule applies the well-known negation-as-failure operator in establishing the truth value of $\mathtt{\neg h\left(F,T\right)}$.  
Similarly, the next two rules establish the truth value of a conjunction {\tt  F}, i.e., {\tt  h(F,T)} is true if none of its conjuncts is false.  
The remaining rules are used to establish the truth value of $\mathtt{h\left(sat\left(C\right),T\right)}$, the satisfaction of concern {\tt \small C} at step {\tt \small T}. Line 8 states that if the formula addressing the concern {\tt \small C} cannot be proven to be true then the concern is not satisfied.  Rules in line 9-10 propagate the unsatisfaction of a concern from its subconcerns. 
Finally, a concern is satisfied if it cannot be proven to be unsatisfied (Line 11).
We can prove the following proposition that relates the implementation and Definition~\ref{def-entailment-state}.

\begin{proposition}[Concern Satisfaction]
        \label{def-concern-satisfiaction}
        For a CPS theory $\Delta=\left(\mathcal{S},I\right)$ and a concern $c$, 
        $c$ is satisfied (or unsatisfied) in $I$ if $h\left(sat\left(c\right),0\right)$ (or $\neg h\left(sat\left(c\right),0\right)$) 
        belongs to every answer set of $\Pi\left(\Delta\right)$,  
        where $\Pi\left(\Delta\right)=\Pi\left(\mathcal{S}\right)^0 \cup \Pi\left(I\right) \cup \Pi_{sat}$. 
\end{proposition}
\begin{proof}
It is easy to see that for any formula $\varphi$ over the fluents in $\mathcal{S}$, the encoding 
and the rules encoding a formula, and the rules in Lines 1--7, $I \models \Lambda(c)$ iff 
$\mathtt{h}(sat(\Lambda(c)^I) ,0)$  belongs to every answer set of $\Pi(\Delta)$ where $\Lambda(c)^I$ is the identifier associated to the 
formula $\Lambda(c)$. Lines 9--10 show that if $c$ has a sub-concern that is not satisfied then it is not satisfied and hence Rule 11 cannot be applied. As such, we have that 
$h(sat(c),0)$ is in an answer set of 
$\Pi(\Delta)$ iff 
the formula 
$\Lambda(c)$ is true and all sub-concerns of $c$ are satisfied in that answer set iff  
$c$ is satisfied in $I$.  
\end{proof}

Since we will be working with the satisfaction of concerns in the following sections, we will therefore need to include 
$\Pi_{sat}$ in $\Pi\left(\mathcal{S}\right)^n$. From now on, whenever we refer to $\Pi\left(\mathcal{S}\right)^n$, we mean  $\Pi\left(\mathcal{S}\right)^n \cup \Pi\left(I\right) \cup \Pi_{sat}$.

\subsection{Computing Most/Least Trustworthy Components}
\label{subsect:computing-most-least}


%
Proposition~\ref{prop:compare} shows that $\succeq_s$ has min/maximal elements, i.e., least/most trustworthy components of a system always exist. The program $\Pi_{mlt}\left(\mathcal{S}\right)$ for computing these components is listed below.
%
%
\begin{lstlisting}[language=clingo,caption=$\Pi_{mlt}$: Computing Most/Least Trustworthy  Components, label=lst:lambda2, mathescape=true,xleftmargin=.01\textwidth, breaklines=true]
r(X,P,C,T) :- comp(X), prop(P), concern(C), step(T), h(active(X,P),T), h(P,T), addBy(C,P).
pos(X,P,C,T) :- r(X,P,C,T), positiveImpact(P,C), h(sat(C),T), step(T).
nPos(X,P,C,T):- r(X,P,C,T), not positiveImpact(P,C), step(T).
nPos(X,P,C,T):- r(X,P,C,T), not h(sat(C),T), step(T).
pos(X,P,C,T) :- pos(X,P,C$_1$,T), subCo(C,C$_1$), step(T).
nPos(X,P,C,T):- nPos(X,P,C$_1$,T), subCo(C,C$_1$), step(T).
twcp(X,TW,T) :- TW=#count{C,P:pos(X,P,C,T), prop(P), concern(C)},  comp(X), step(T).
twcn(X,TW,T) :- TW=#count{C,P:nPos(X,P,C,T), prop(P), concern(C)}, comp(X), step(T).
higher(X$_1$,X$_2$,T) :- twcp(X$_1$,TWp1,T), twcp(X$_2$,TWp2,T), twcn(X$_1$,TWn1,T), twcn(X$_2$,TWn2,T), d$_1$=TWp1/(TWn1 + 1), d$_2$=TWp2/(TWn2 + 1), d$_1$ > d$_2$, step(T), TWp1!=0, TWp2!=0.
higher(X$_1$,X$_2$,T):- step(T), twcp(X$_1$,0,T), twcp(X$_2$,0,T), twcn(X$_1$,TWn1,T), twcn(X$_2$,TWn2,T), TWn1 < TWn2.
equal(X$_1$,X$_2$,T)  :- twcp(X$_1$,TWp1,T), twcp(X$_2$,TWp2,T), twcn(X$_1$,TWn1,T), twcn(X$_2$,TWn2,T), d$_1$=TWp1/(TWn1 + 1), d$_2$=TWp2/(TWn2 + 1), d$_1$ = d$_2$, step(T), TWp1!=0, TWp2!=0.
equal(X$_1$,X$_2$,T) :- step(T), twcp(X$_1$,0,T), twcp(X$_2$,0,T), twcn(X$_1$,TWn1,T), twcn(X$_2$,TWn2,T), TWn1=TWn2.
not_highestTW(X$_2$,T) :- comp(X$_1$), comp(X$_2$), higher(X$_1$,X$_2$,T), step(T).
not_lowestTW(X$_1$,T)   :- comp(X$_1$), comp(X$_2$), higher(X$_1$,X$_2$,T), step(T).
most(X,T)  :- comp(X), not not_highestTW(X,T), step(T).
least(X,T) :- comp(X), not not_lowestTW(X,T), step(T).
\end{lstlisting}

In Listing~\ref{lst:lambda2}, {\tt \small addBy(C,P)} and {\tt \small positiveImpact(P,C)} 
are defined in the program $\Pi(\Omega)$ (Subsection~\ref{sub-section:bg_asp_representation}). {\tt \small addBy(C,P)} is true means that a property $P$ addresses a concern $C$. 
{\tt \small positiveImpact(P,C)} is true means that the satisfaction of property $P$ impacts positively on the satisfaction of concern $C$. The predicate {\tt \small r(X,P,C,T)} (Line 1) encodes the relationship between $X$, $P$ and $C$ at the time $T$ which says that the component $X$ is working with the property $P$ at time $T$ and $P$ addresses concern $C$. 
The second rule (Line 2) defines the predicate {\tt \small pos(X,P,C,T)} that encodes the 
positive affected relationship between component $X$ and concern $C$ at time step $T$ through property $P$ which is true  if the concern $C$ is satisfied and {\tt \small positiveImpact(P,C)} and {\tt \small r(X,P,C,T}) hold.
Lines 3--4 define {\tt \small nPos(X,P,C,T)}, which holds at time $T$ if {\tt \small r(X,P,C,T)} holds but either {\tt \small positiveImpact(P,C)} is not defined in $\Omega$ or concern $C$ is not satisfied. This element is used for the computation of the denominator of Equation~\eqref{def-twc}. The rest of the listing defines the relationship {\tt \small higher} between components encoding the $\succeq_T$ where $T$ represents the state at the time $T$ of the system and identifying the most and least trustworthy components. 
Lines 5--6 propagate the \emph{positive affected} and \emph{impact} relations ({\tt \small pos/4}, {\tt \small nPos/4}) of a concern from its subconcerns. $twcp\left(x,tw,t\right)$ (resp. $twcn\left(x,tw,t\right)$) encodes the number of concerns positively affected (resp. impacted) by component $x$ at step $t$. The atom 
$\#\mathtt{count\{C,P:pos\left(X,C,P,T\right),prop\left(P\right),concern\left(C\right)\}}$ is an aggregate atom in ASP and 
encodes the cardinality of the set of all concerns positively impacted by $P$ and $X$.

We can show that 
the following proposition holds. 

\begin{proposition} 
        \label{prop:2}
        For a CPS theory $\Delta = \left(\mathcal{S},I\right)$ and  an answer set $S$ of program $\Pi\left(\mathcal{S}\right)^n \cup \Pi\left(I\right) \cup \Pi_{mlt}$, 
        if $most\left(x,t\right) \in S$ (resp. $least\left(x,t\right) \in S$) 
        then $x$ is a most (resp. least) trustworthy component in the state $s_t$. 
        
\end{proposition}


The proof follows immediately from the definition of the predicate $addBy$, $positiveImpact$ and the definition of aggregate functions in ASP. As such, to identify the most trustworthy component of $\mathcal{S}$, we only need to compute an answer set $S$ of $\Pi\left(\mathcal{S}\right)^n \cup \Pi\left(I\right) \cup \Pi_{mlt}$ and use Proposition~\ref{prop:2}.

\begin{example} 
\label{example:most_trustworthy_1}
Consider the $\mathcal{S}_{lkas}$ domain.          
\begin{itemize}
    \item 
        Let us consider the initial configuration $I^{1}_{lkas}$ of LKAS system  where every properties are observed to be true. 
        For $\Delta_{lkas} = \left(\mathcal{S}_{lkas},I^{1}_{lkas}\right)$,  
        we can easily see that (from Figure~\ref{fig:trustworthiness_2}) the atoms:    
        $pos\left(cam,advanced\_mode,integrity,0\right)$,  
        $pos\left(cam,secure\_boot,cyber\_security,0\right)$, etc. belong to every answer set of  $\Pi\left(\Delta_{lkas}\right) =\Pi\left(\mathcal{S}_{lkas}\right)^n \cup \Pi\left(I^{1}_{lkas}\right) \cup \Pi_{mlt}^{lkas} $. Similar atoms are present to record the number of concerns affected by different properties.
        Furthermore, 
        $twcp\left(cam,28,0\right)$, $twcn\left(cam,6, 0\right)$, 
        $twcp\left(sam,40,0\right)$, $twcn\left(sam,0, 0\right)$, $twcp\left(bat,6,0\right)$ and $twcn\left(bat,5,0\right)$ belong to any answer set of  $\Pi\left(\mathcal{S}_{lkas}\right)^n \cup \Pi\left(I^{1}_{lkas}\right) \cup \Pi_{mlt}^{lkas}$: $SAM$ is the most trustworthy component; $BAT$ is the least trustworthy components at step 0.

    \item
        Now, let us consider $I^{2}_{lkas}$ of LKAS system ( Figure~\ref{fig:trustworthiness_2}) where there are two properties that are observed to be \emph{False}: {\tt \small Firewall-Setup} and {\tt \small Trusted-Auth-Device}. 
        For $\Delta_{lkas} = \left(\mathcal{S}_{lkas},I^{2}_{lkas}\right)$, the computation of the program $\Pi\left(\mathcal{S}_{lkas}\right)^n \cup \Pi\left(I^{2}_{lkas}\right) \cup \Pi_{mlt}^{lkas}$ shows us: 
        $twcp\left(cam,22,0\right)$, $twcn\left(cam,6, 0\right)$, 
        $twcp\left(sam,22,0\right)$, $twcn\left(sam,12, 0\right)$, $twcp\left(bat,0,0\right)$ and $twcn\left(bat,11,0\right)$ belong to any answer set of  $\Pi\left(\mathcal{S}_{lkas}\right)^n \cup \Pi\left(I^{2}_{lkas}\right) \cup \Pi_{mlt}^{lkas}$. In this situation, $CAM$ is the most trustworthy component; $BAT$ is the least trustworthy components at step 0.
\end{itemize}        
        
\end{example}
We conclude this part with a brief discussion on possible definitions of $\succeq$. The proposed definition assumes everything being equal (e.g. all concerns and properties are equally important, the roles of every components in a CPS system are equal, etc.). 
In practice, the ordering $\succeq$ might be qualitative and user-dependent, e.g., an user might prefer confidentiality over integrity. $\succeq$ can be defined over a qualitative ordering and implemented in ASP in a similar fashion that preferences have been implemented (e.g., \cite{GelfondS98}).

%

\subsection{Computing Mitigation Strategies}
\label{sub-section:compute_mitigation_strategy}

The program $\Pi\left(\mathcal{S}\right)^{n} \cup \Pi_{sat}$ can be for computing a mitigation strategy  
by adding the rules shown in Listing~\ref{lst:generating_plan}:  
\begin{lstlisting}[language=clingo,caption={$\Pi^{n}_{plan}$: Generating Plan},  label=lst:generating_plan, mathescape=true,xleftmargin=.01\textwidth, breaklines=true]
1{occurs(A,T):action(A)}1 :- step(T), T<n.
:- occurs(A,T), not exec(A,T).
:- not h(sat(c), n). 
\end{lstlisting}
The first rule containing the atom $\mathtt{1 \{occurs\left(A,T\right):action\left(A\right)\}1}$ --- a choice atom --- is intuitively used to generate the action occurrences and says that at any step $T$, exactly one action must occur. The second rule states that an action can only occur if it is executable. The last rule helps enforce that $h\left(sat\left(c\right),n\right)$ must be true in the last state, at step $n$. 
For a set of concerns $\Sigma$, let $\Pi^{n}_{plan}\left[\Sigma\right]$ be the program obtained from $\Pi^{n}_{plan}$ by replacing its last rule with the set $\mathtt{\{{:}{-} not \quad h\left(sat\left(c\right),n\right). \mid c \in}$ $\Sigma\}$. 
Based on the results in answer set planning, we can show: 
\begin{proposition}
    \label{prop:mitigation}
        Let $\Delta = \left(\mathcal{S},I\right)$ be a CPS theory and $\Sigma$ be a set of concerns in $\Omega$. Then, $\left[a_0,\ldots,a_{n-1}\right]$ is a mitigation strategy for $\Sigma$ iff 
        $\Pi\left(\Delta\right) \cup \Pi^{n}_{plan}\left[\Sigma\right]$ has an answer set $S$ such that 
        $occurs\left(a_i,i\right) \in S$ for every $i=0,\ldots,n-1$.
\end{proposition}
The proof of this proposition relies on the properties of $\Pi\left(\Delta\right)$ discussed in previous section and the set of constraints in $\Pi^{n}_{plan}\left[\Sigma\right]$.

\subsection{Non-compliance Detection in CPS Systems}

The program $\Pi\left(\mathcal{S}\right)^{n} \cup \Pi_{sat}$ can be used in non-compliance detection by adding the rules shown in Listing~\ref{lst:conflict_detection}:  
\begin{lstlisting}[language=clingo,caption={$\Pi^{n}\left(SA,SC\right)$: Non-compliance  Detection},  label=lst:conflict_detection, mathescape=true,xleftmargin=.01\textwidth, breaklines=true]
1{occurs(A,T):sa_action(A)}1 :- step(T), T<n, not conflict(T).
:- occurs(A,T), not exec(A,T), step(T).
1{h(F,0); $\neg$h(F,0)}1 :- fluent(F).    
conflict(T)   :- sc_concern(C), $\neg$h(sat(C),T), step(T).
conflict(T+1) :- conflict(T), step(T). 
:- not conflict(n). 
\end{lstlisting}
The first two rules are similar to the rules for the planning program, with the exception that the action selection focuses on the actions in the set $SA$.
The third rule generates an arbitrary initial state. 
The rules 4-5 state that if some concern in $SC$ is not satisfied at time $T$ then a conflict arises and the constraint on the last rule says that we would like to create a conflict at step $n$. 

We assume that actions in $SA$ are specified by atoms of the form $sa\_action\left(a\right)$ and concerns in $SC$ are specified by atoms of the form $sc\_concern\left(c\right)$. It is easy to see that an answer set $S$ of $\Pi\left(\mathcal{S}\right)^{n} \cup \Pi_{sat} \cup \Pi^{n}\left(SA,SC\right)$ represents a situation in which the system will eventually not satisfy some concern in $SC$. Specifically, if the sequence of actions $\left[a_0,\ldots,a_t\right]$ such that $occurs\left(a_i, i\right) \in S$ and, for $s>t$, there exists no $occurs\left(a_s,s\right)\in S$, is executed in the 
initial state (the set $\{f \mid h\left(f,0\right) \in S, f \in F\} \cup \{\neg f \mid \neg h\left(f,0\right) \in S, f \in F\}$) then some concern in $SC$ will not be satisfied after $n$ steps. In other words, to check whether $\mathcal{S}$ is  weakly $n$-noncompliant, we only need to check whether 
$\pi_n = \Pi\left(\mathcal{S}\right)^{n} \cup \Pi_{sat} \cup \Pi^{n}\left(SA,SC\right)$ as an answer set of not. The proof of this property relies on the definition of an answer set for a program with constraints, which say that the constraint {\tt \small :- not conflict(n).} must be false in the answer set, which in turn implies that {\tt \small conflict(n)} must be true. 

If $\mathcal{S}$ is weakly $n$-noncompliant, we can do one more check to see whether it is strongly $n$-complaint as follows. Let  $\pi_n'$ be a program obtained from $\pi_n$ by replacing ``{\tt \small :- not conflict(n)}'' 
with ``{\tt \small :- conflict(n).}'' We can show that if $\pi_n'$ has no answer set then for every initial state of $\mathcal{S}$ no action sequence is executable or there exists some action sequence such that {\tt \small conflict(n).} is true. Combining with the fact that $\mathcal{S}$ is weakly $n$-noncompliant, this implies that the domain is strongly $n$-noncompliant. Again, the proof of this property relies on the definition of answer sets of programs with constraints, which say that the constraint {\tt \small :- conflict(n).} must be false in an answer set, which in turn implies that {\tt \small conflict(n)} must be false. However, the program having no answer set implies that every executable sequence of actions will generate {\tt \small conflict(n)}.

\subsection{Likelihood of Concerns Satisfaction and Preferred Mitigation Strategies}
In this subsection, we present an ASP program for computing LoS of concerns and preferred mitigation strategies using LoS. 
Listing~\ref{lst:compute_likelihood_sat} shows the ASP encoding for computing of LoS of concerns. It defines 
the predicate {\tt \small llh\_sat(C,N,T)} which states that the likelihood of satisfaction of concern $C$ at time step $T$ is $N$. It starts with the definition of different predicates {\tt \small nAllPosCon/3} and {\tt \small nActPosCon/3} representing  $rel^\textnormal{+}\left(c\right)$ and $rel_{sat}^\textnormal{+}\left(c,s\right)$ at the step $T$, i.e., the number of all possible positively impacting properties on concern $C$ and the number of positively impacting properties on concern $C$ holding in step $T$, respectively. Recall that 
{\tt \small positiveImpact(P,C)} is defined as in Subsection~\ref{subsect:computing-most-least}. 
Line 5 creates an ordering between subconcerns of concern $C$ for the computation of {\tt \small llh\_sat(C,N,T)}. The LoS for a concern without a subconcern is computed in Line 8. Rules on the lines 9-12 compute the LoS of concerns in accordance with the order created by rule on Line 1. {\tt \small llh\_sat(C,N,T)} is then computed using Equation~\ref{def-llh}.
%


\begin{lstlisting}[language=clingo,caption=$\Pi_{LoS}$: Computing Likelihood of Concerns Satisfaction, label=lst:compute_likelihood_sat, mathescape=true,xleftmargin=.01\textwidth, breaklines=true]
nAllPosCon(C,N2,T):- concern(C), step(T), N2=#count{P,Com : comp(Com), prop(P), positiveImpact(P,C), addBy(C,P), relation(Com,P)}.
nActPosCon(C,N1,T):- concern(C), step(T), N1=#count{P,Com : comp(Com), prop(P), positiveImpact(P,C), addBy(C,P), relation(Com,P),       h(active(Com,P),T)}.
deg_pos(C,1,T)         :- step(T), concern(C), nAllPosCon(C,0,T).
deg_pos(C,N1*100/N2,T) :- nAllPosCon(C,N2,T), nActPosCon(C,N1,T), concern(C), N2!=0.
order(SC,C,N) :- subCo(C,SC), N={SC < SCp : subCo(C,SCp)}.
hSubCo(C) :- subCo(C,SC), concern(C), concern(SC).
$\neg$hSubCo(C):- concern(C), not hSubCo(C).
llh_sat_sub(C,1,T) :- step(T), concern(C), $\neg$hSubCo(C). 
llh_sat(C,N1*N2,T) :- step(T), concern(C), llh_sat_sub(C,N1,T), deg_pos(C,N2,T).
llh_sat_sub_aux(C,0,X,T) :- step(T), subCo(C,SC), order(SC,C,0), llh_sat(SC,X,T).
llh_sat_sub_aux(C,N,X*Y,T) :- step(T), subCo(C,SC), order(SC,C,N), llh_sat(SC,Y,T), llh_sat_sub_aux(C,N-1,X,T). 
llh_sat_sub(C,X,T) :- llh_sat_sub_aux(C,N,X,T), step(T), concern(C), not llh_sat_sub_aux(C,N+1,_,T).
\end{lstlisting}
%

 %
%
 
%
It is easy to check that the above program correctly computes the values of $deg^\textnormal{+}\left(c,s\right)$ and $\varphi_{LoS}\left(c,s\right)$. Indeed, the program $\Pi\left(\Delta_{lkas}\right) = \Pi\left(\mathcal{S}_{lkas}\right)^n \cup \Pi\left(I_{lkas}\right) \cup \Pi^{c}_{lkas} \cup \Pi_{sat} \cup \Pi^{n}_{plan} \cup \Pi_{LoS}$ correctly computes the LoS of concerns for various concerns as shown in Subsection~\ref{sub-section:compute_mitigation_strategy_theory} (Figure~\ref{fig:trust_tree_LoSC}).

Having computed LoS of concerns and $\varphi_{LoS}$, identifying the best strategies in according to the two approaches in Subsection~\ref{sub-section:compute_mitigation_strategy_theory} is simple. We only need to add rules that aggregates the LoS of the top-level concerns specified in the CPS with their corresponding weights or preferences. This is done as follows: 

\begin{itemize} 

    \item \emph{Weighted LoS}: Listing~\ref{lst:compute_weighted_los} computes the weighted LoS of the final state. The rule is self-explanatory. 
\begin{lstlisting}[language=clingo,caption=Computing Weighted LoS, label=lst:compute_weighted_los, mathescape=true,xleftmargin=.01\textwidth, breaklines=true]
scoreLoS(Sc,T) :- llh_sat(functionality,V$_{fun}$,T), wLoS(functionality,W$_{fun}$), llh_sat(business,V$_{bus}$,T), wLoS(business,W$_{bus}$), llh_sat(human,V$_{hum}$,T), wLoS(human,W$_{hum}$), llh_sat(trustworthiness,V$_{tru}$,T), wLoS(trustworthiness,W$_{tru}$), llh_sat(timing,V$_{tim}$,T), wLoS(timing,W$_{tim}$), llh_sat(data,V$_{dat}$,T), wLoS(data,W$_{dat}$), llh_sat(boundaries,V$_{bou}$,T), wLoS(boundaries,W$_{bou}$), llh_sat(composition,V$_{com}$,T), wLoS(composition,W$_{com}$), llh_sat(lifestyle,V$_{lif}$,T), wLoS(lifestyle,W$_{lif}$), Sc = V$_{fun}$*W$_{fun}$ + V$_{bus}$*W$_{bus}$ + V$_{hum}$*W$_{hum}$ + V$_{tru}$*W$_{tru}$ + V$_{tim}$*W$_{tim}$ + V$_{dat}$*W$_{dat}$ + V$_{bou}$*W$_{bou}$ + V$_{com}$*W$_{com}$ + V$_{lif}$*,W$_{lif}$.
\end{lstlisting}

        \item \emph{Specified Preferences LoS}: ASP solver provides a convenient way for computing preferences based on lexicographic order among elements of a set. Assume that {\tt \small Trustworthiness} is preferred to {\tt \small Business} then the two statements 
        
        \hspace*{1cm} {\tt \small \#\textbf{maximize}\{$V_1$@k : llh\_sat(trustworthiness, $V_1$, n)\}} \\  
        \hspace*{1cm} {\tt \small \#\textbf{maximize}\{$V_2$@k': llh\_sat(business, $V_2$, n)\}}   
    
        with $k > k'$ and $n$ is the length of the plan will return answer sets in the lexicographic order, preferring  the concern 
        {\tt \small Trustworthiness} over {\tt \small Business}. With these statements, any specified preferred LoS over the set of top-level concern can be implemented easily.   
    
\end{itemize}

\subsection{Computing Mitigation Strategy with The Best Chance to Succeed}
\label{sub-section:for_probability_success_imp}

To compute strategies with the maximal probability of success, we only need to extend the program $\Pi^n_{plan}$  with the following rules: %
\begin{itemize}  
        \item for each statement ``$a \quad \prob{} \quad v \quad \mathbf{if} \quad p_1,\ldots,p_n$'', the two rules: \\
        \hspace*{1cm} $\mathtt{ pr\left(a,v,T\right)\:\: {:}{-} \:\: h^*\left(p_1,T\right),\ldots,h^*\left(p_n,T\right).}$ \\
        \hspace*{1cm} $\mathtt{ dpr\left(a,T\right)\:\: {:}{-} \:\: h^*\left(p_1,T\right),\ldots,h^*\left(p_n,T\right).}$ \\
        which check for the satisfaction of the condition in a statement defining the probability of success in the step $T$ and states that it is defined. 
        \item the rule: \\  
        \hspace*{1cm} $\mathtt{ pr\left(A,1,T\right)\:\: {:}{-} \:\: exec\left(A,T\right), \naf \ dpr\left(A,T\right).}$ \\
        which says that by default, the probability of success of $a$ at step $T$ is 1. 
        
        \item computing the probability of the state at step $T$: \\
        \hspace*{1cm} $\mathtt{ prob\left(1,0\right).}$ \\
        \hspace*{1cm} $\mathtt{ prob\left(U*V,T\textnormal{+}1\right)\:\: {:}{-} \:\: prob\left(U,T\right), occurs\left(A,T\right), pr\left(A,V,T\right).}$
        \\
        where the first rule says that the probability of the state at the time 0 is 1; $prob\left(v,t\right)$ states that the probability of reaching the state at the step $t$ is $v$ and is computed using the second rule. 
\end{itemize}
Let $ \Pi^{n}_{bestPrS}$ be $\Pi^n_{plan}$ and the above rules.  
We have that if $\left[a_0,\ldots,a_{n-1}\right]$ and $S$ is an answer set of $\Pi\left(\Delta\right) \cup \Pi^{n}_{bestPrS} \cup \{occurs\left(a_i,i\right) \mid i=0,\ldots,n-1\}$ then $prob\left(\Pi_{i=0}^{n-1} pr\left(a_i,s_i\right),n\right) \in S$.
To compute the best strategy, we add the rule \\
\hspace*{1cm} $\mathtt{\small \#maximize \{V : prob\left(V,n\right)\}.}$ \\
to the program $\Pi^{n}_{bestPrS}$. 
\begin{example}
        Continue with Example~\ref{example:bg_lkas_and_extend} after a cyber-attack occurs and causes the property {\tt \small basic-mode} to be \emph{False}. As in Section~\ref{sub-section:compute_mitigation_strategy}, the five mitigation strategies ($\alpha_1,\alpha_2,\alpha_3,\alpha_4$ and $\alpha_5$) are generated to restore the LKAS system.
        Assume that the probability of success of $\mathtt{\small tOn\left(basic\_mode\right)}$,
        $\mathtt{\small switM\left(cam,advanced\_mode\right)}$, and 
        $\mathtt{\small switM\left(sam,advanced\_mode\right)}$ are 
        0.2, 0.6, 0.7 in every state, respectively. 
        In this case, the strategies $\alpha_2$ and $\alpha_3$ have the maximal probability to succeed. 
\end{example}
%
\section{Towards a Decision-Support System for CPSF}
As a demonstration of the potential use of our approach, in this section we give a brief overview of a decision-support system that is being built for use by CPS designers, managers and operators. We also include preliminary considerations on performance aspects.
\begin{figure}[h]
	\includegraphics[width=1.0\textwidth]{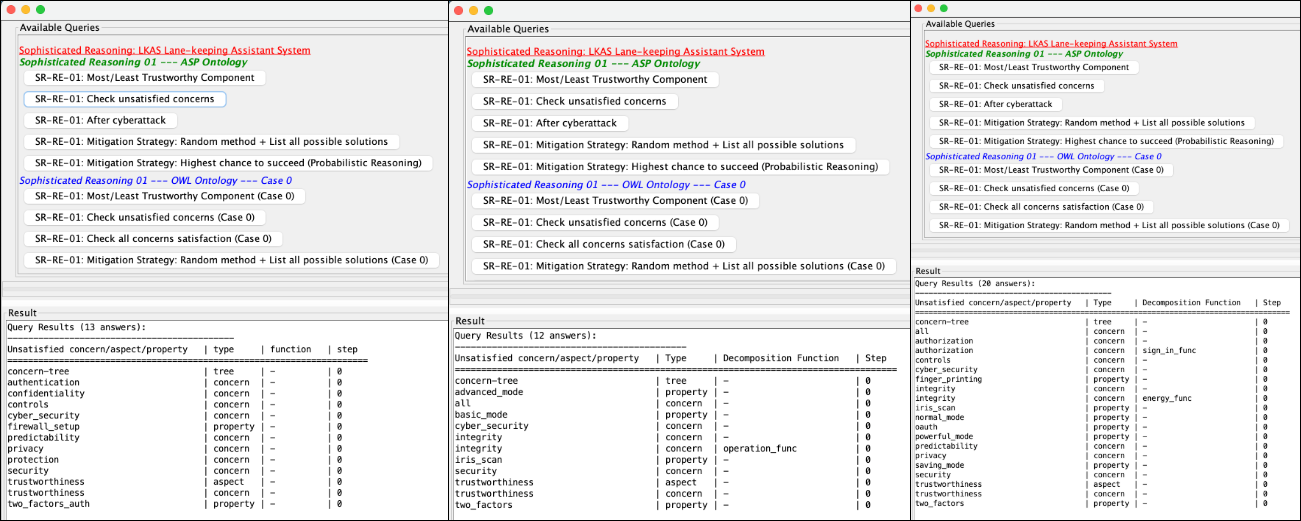}
	\caption{Computing Satisfaction of Concerns in Reasoning Component\label{fig:cps_reasoner_satisfaction}}
\end{figure}
The decision-support system relies on an ASP-based implementation for reasoning tasks in CPS theories (described in Section~\ref{section:implementation}) with the different modules for answering queries described in Section~\ref{sub:section:reasoning_tasks}, and comprises a \textit{reasoning component} and a \textit{visualization component}.
Figure~\ref{fig:cps_reasoner_satisfaction} shows the \textit{reasoning component} at work on computing satisfaction of concerns related to the LKAS domain example (described in Section~\ref{sub-section:concern_satisfaction}). Figure~\ref{fig:cps_reasoner_mlt_ms_llh} illustrates the reasoning component at work on other modules (Section~\ref{subsect:computing-most-least}--~\ref{sub-section:for_probability_success_imp}) with different situations related to the LKAS domain. Notice how the user can ask the system to reason about satisfaction of concerns, to produce mitigation plans as well as to select the most preferred mitigation strategy, etc.
\begin{figure}[h]
	\includegraphics[width=1.0\textwidth]{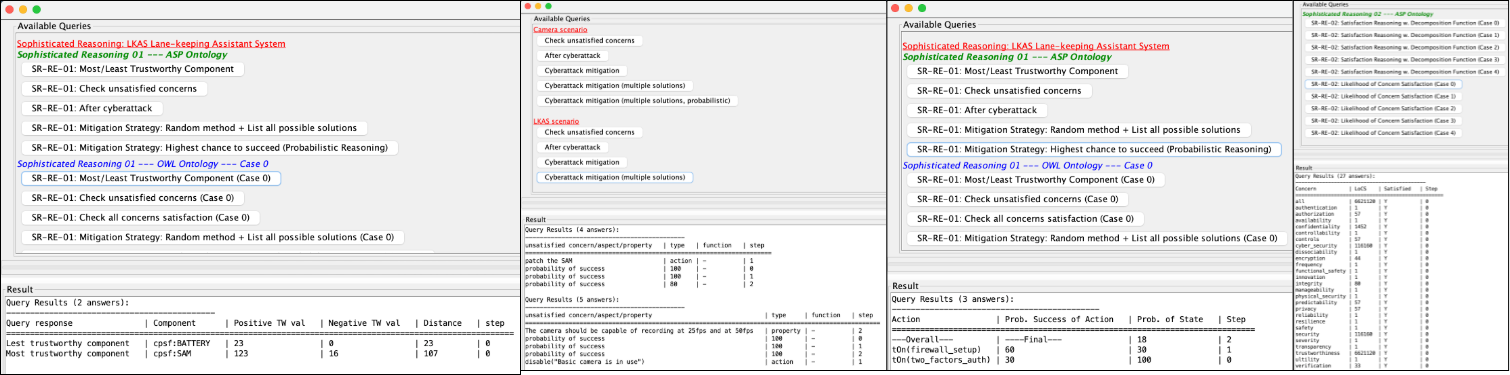}
	\caption{Other reasoning modules in Reasoning Component\label{fig:cps_reasoner_mlt_ms_llh}}
\end{figure}

The output of the reasoning component can then be fed to the \textit{visualization component}, where advanced visualization techniques allow practitioners to get a birds-eye view of the CPS or dive into specific details. For instance, the sunburst visual from Figure~\ref{fig:visualizer_example} provides a view of the CPS from Figure~\ref{fig:trustworthiness_2} where the aspects are presented in the inner most ring. Moving outwards, the visualization shows concerns from increasingly deeper parts of the concern tree and properties. The left-hand side of the figure depicts the visualization in the case in which all concerns are satisfied (blue), while the right-hand side shows how the sunburst changes when certain concerns (highlighted as red) are not satisfied. Focusing on the right-hand side, the text box open over the visual reports that the trustworthiness aspect is currently not not satisfied and the level at which this concern is not being met is the concern of privacy and the property of manageability. The visual allows for a pinpoint where within the CPS framework issues have arisen that when addressed can enable a working state. We omit the details of \emph{visualization component} description as it is not the focus of this paper.
\begin{figure}[h]
	\includegraphics[width=0.98\textwidth]{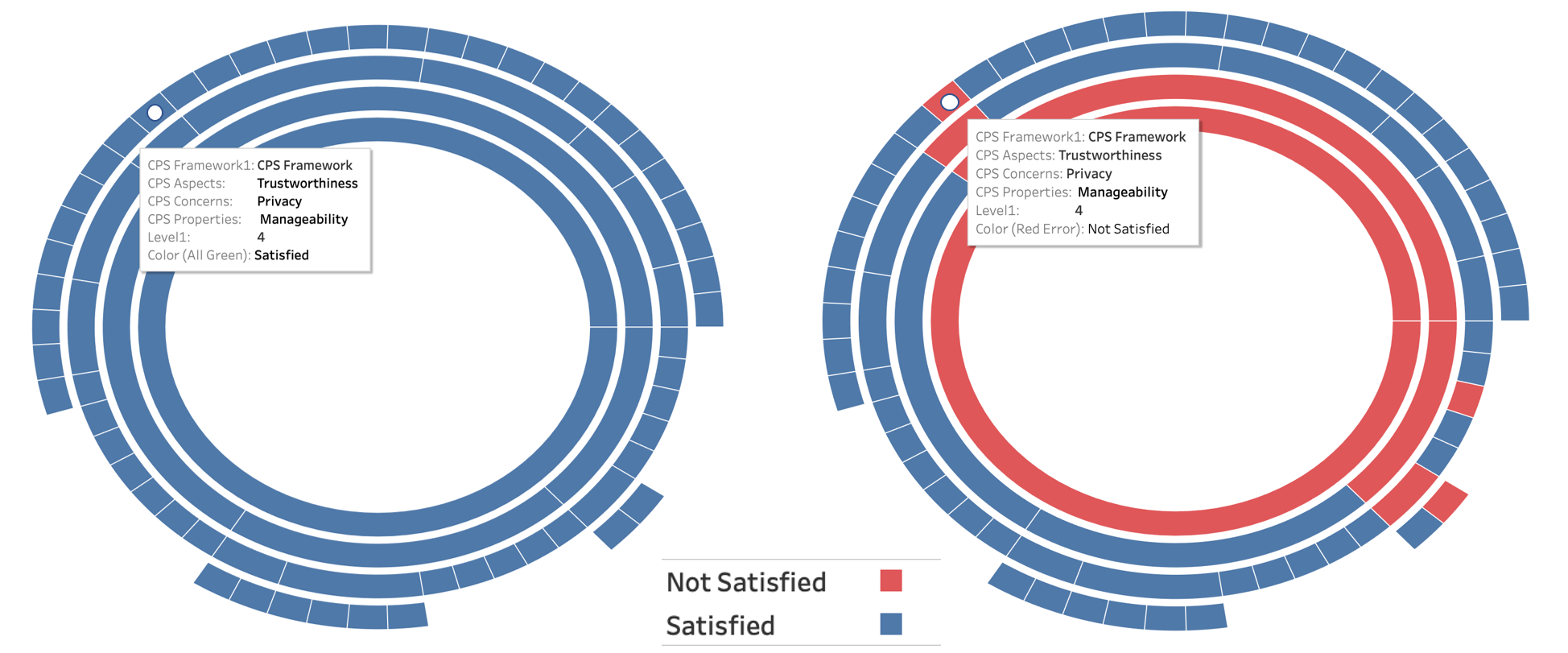}
	\caption{Visualization component\label{fig:visualizer_example}}
\end{figure}

To ensure flexibility and to allow for investigation on the scalability on larger CPS, the decision-support system is designed to support a variety of hybrid ontology-ASP reasoning engines. Currently, we consider four reasoning engines: the \emph{na\"ive engine} is implemented by connecting, in a loosely-coupled manner\footnote{By loosely-coupled, we mean that the components see each other as black-boxes and only exchange information, via simple interfaces, at the end of their respective computations. Compare this with a tightly-coupled architecture, where the components have a richer interfaces for exchange state information and controlling each other's execution flow while their computations are still running.}, the SPARQL reasoner{\footnote{https://www.w3.org/TR/rdf-sparql-query/}}
and the Clingo ASP solver. This engine issues a single SPARQL query to the ontology reasoner at the beginning of the computation, fetching all necessary data. The \emph{Clingo-Python engine} is another loosely-coupled engine, leveraging Clingo's ability to run Python code at the beginning of the computation. This engine issues multiple queries in correspondence to the occurrences of special ``external predicates'' in the ASP\ program, which in principle allows for a more focused selection of the content of the ontology. The \emph{DLVHex2 engine} also uses a similar fragmentation of queries, but the underlying solver allows for the queries to be executed at run-time, which potentially results in more focused queries, executed only when strictly needed. Finally, the \emph{Hexlite engine} leverages a similar approach, but was specifically designed as a smaller, more performant alternative to \emph{DLVHex2}.

In this preliminary phase of our investigation on scalability, all reasoning engines have exhibited similar performance, as exemplified by Table \ref{table:compare_reasoner_new}. 
The table summarizes the results of question-answering experiments on the Lane Keeping/Assist System (LKAS) domain and on the Smart Elevator domain~\cite{Nguyen2020ReasoningAT}. The reasoning tasks considered are for answering queries discussed earlier, including:
\begin{itemize}
    \item ({\tt$\mathbf{Q}_1$}) Computing satisfaction of concerns.
    \item ({\tt$\mathbf{Q}_2$}) Computing most/least trustworthy components.
    \item ({\tt$\mathbf{Q}_3$}) Generating mitigation strategies.
    \item ({\tt$\mathbf{Q}_4$}) Non-compliance detection in a CPS.
    \item ({\tt$\mathbf{Q}_5$}) Selecting the best mitigation strategy by preferred mitigation strategies.
    \item ({\tt$\mathbf{Q}_6$}) Computing the likelihood of concerns satisfaction.
\end{itemize}
In Table~\ref{table:compare_reasoner_new}, the performance of the execution for each query ({\tt$\mathbf{Q}_1$}-{\tt$\mathbf{Q}_6$}){\footnote{We use a Macbook Pro 16 running macOS Big Sur Version 11.5.2, 32GB RAM DDR4, 2.6Ghz 6-Core Intel Core i9, and ASP solver {\tt \small Clingo}}} is measured by the average processing time of reasoning computations in our experiment CPS theories (LKAS and Smart Elevator) with different initial situations (different initial configurations).
While the results show that the na\"ive engine is marginally better than the others, the differences are quite negligible, all within $10\%$.
\begin{table}[h]
	\begin{center}
		\begin{tabular}{ c r r r r r r r r }
			\topline
			\multirow{2}{*}{\textbf{Reasoning Tasks}} & \multicolumn{4}{c}{\textbf{LKAS Domain}} & %
			\multicolumn{4}{c}{\textbf{Smart Elevator Domain}} \\
			\cmidrule(lr){2-5}
            \cmidrule(lr){6-9}
			& \specialcell[c]{Na\"ive} & \specialcell[c]{Clingo\\-Python} & {DLVHex2} & {Hexlite} & \specialcell[c]{Na\"ive} & \specialcell[c]{Clingo\\-Python} & {DLVHex2} & {Hexlite} \\
			\hline
			{\tt $\mathbf{Q}_1$} & 1.35s & 1.48s & 1.32s & 1.37s & 1.31s & 1.45s & 1.30s & 1.35s \\
			\hline
			{\tt $\mathbf{Q}_2$} & 1.28s & 1.43s & 1.29s & 1.32s & 1.25s & 1.32s & 1.22s & 1.30s \\
			\hline
			{\tt $\mathbf{Q}_3$} & 1.36s & 1.52s & 1.38s & 1.41s & 1.33s & 1.49s & 1.37s & 1.39s \\
			\hline
			{\tt $\mathbf{Q}_4$} & 1.41s & 1.52s & 1.41s & 1.45s & 1.40s & 1.53s & 1.41s & 1.47s \\
			\hline
			{\tt $\mathbf{Q}_5$} & 1.38s & 1.47s & 1.42s & 1.39s & 1.26s & 1.39s & 1.33s & 1.35s \\
			\hline
			{\tt $\mathbf{Q}_6$} & 1.74s & 1.93s & 1.79s & 1.81s & 1.78s & 1.95s & 1.77s & 1.86s \\
			\botline
		\end{tabular}
	\end{center} 
	\vspace*{.001in}
	\caption{CPS domains Querying, Extracting and Reasoning Summary}
	\label{table:compare_reasoner_new}
	\vspace*{-.1in}
\end{table}
It is conceivable that larger-scale experiments will eventually exhibit similar patterns to those found in other research on the scalability of hybrid systems (e.g., \cite{bl16}). A thorough analysis will be the subject of a separate paper where we have done some preliminary experiment with our CPS reasoning system and found that it can work ontologies with more than 150K triples, 85 classes, 61K individuals, 30 object properties, 40 data properties, and 45 subclass relations within a minute.

\section{Related Work}
Due to the difference in level of abstraction, most of the approaches from the literature can be viewed as orthogonal and complementary to ours. Thus, we focus our review of related work on what we consider to be the most relevant approaches. 

The literature from the area of cybersecurity is often focused on the notion of graph-based attack models. Of particular relevance is the work on Attack-Countermeasure Trees (ACT) \cite{DBLP:journals/scn/RoyKT12}. An ACT specifies how an attacker can achieve a specific goal on a IT system, even when mitigation or detection measures are in place. While ACT are focused on the Cybersecurity concern, our approach is rather generally applicable to the broader Trustworthiness aspect of CPS and can in principle be extended to arbitrary aspects of CPS and their dependencies. The underlying formalization methodology also allows for capturing sophisticated temporal models and ramified effects of actions. In principle, our approach can be extended to allow for quantitative reasoning, e.g., by leveraging recent work on Constraint ASP and probabilistic ASP \cite{bl16,os12,BaralGR05}. As we showed above, one may then generate answers to queries that are \emph{optimal} with respect to some metrics. It is worth pointing out that the combination of physical (non-linear) interaction and logical (discrete or Boolean) interaction of CPS can be modeled as a mixed-integer, non-linear optimization problem (MINLP) extended with logical inference. MINLP approaches can support a limited form of logic, e.g., through disjunctive programming \cite{balas:1975}. But these methods seem to struggle with supporting richer logics and inferences such as ``what-if'' explorations. For relevant work in this direction, we refer the reader to \cite{ruthOnline17,DBLP:journals/corr/DIddioH17}.

One major focus in the area of cybersecurity is the identification and mitigation of compromised devices. Behavior analysis and behavioral detection are some of the approaches used in this area. \cite{bau19} proposes a system-level framework for the identification of compromised smart grid devices. The approach employs a combination of system call and function call tracing, which are paired with signal processing and statistical analysis. In a similar vein, \cite{scw18} covers model-based techniques for addressing the problem of sensors that can be manipulated by an attacker.  It is worth noting that,  in our methodology, the presence or lack of compromised devices or components -- and even the type of compromise -- can be captured by means of properties, which in turn affect specific concerns. Techniques such as those described in the cited papers can then be used to determine whether such properties are satisfied or not.

Another related, complementary approach is presented in \cite{arm17}, where the authors tackle the problem of validation and verification of requirements. The paper proposes model-based testing as a solution to two key problems in validation and verification of requirements: translating requirements into concrete test inputs and determining what the outcome of such tests says about the satisfaction of the requirements. From this point of view, the approach from \cite{arm17} can be used to provide the information about satisfaction of requirements that is necessary for the reasoning tasks covered by in our investigation.

\cite{lee16} analyzes the role of models in the engineering of CPS and argues for classes of models that trade accuracy and detail in favor of simplicity and clarity of semantics. This idea is in line with the considerations that prompted the development of {\cpsf}, and which are infused in our work through its legacy. In a related fashion, \cite{row19} proposes a survey of \emph{conformance relations}, where the term describes the link between functional behavior of a model and the behavior of the implemented system (or of a more concretized model). Conformance relations are typically applied to the task of analyzing requirements and their link to the CPS being modeled, and in that sense \cite{row19} is orthogonal to our work. On the other hand, the paper elicits the interesting issue of whether the characterization of CPS from {\cpsf} might be viewed, itself, as a conformance relation. This is an open question, which we plan to address in the future.

\cite{thr19} presents a rich survey of frameworks for implementing reasoning mechanisms in smart CPS. It is to be noted that the focus of the paper is on the reasoning mechanisms that occur within a CPS in order to achieve ``smartness'' \cite{thr19}, while our focus is on reasoning mechanisms that allow designers, maintainers and operators to reason about a CPS -- where the CPS itself may or may not be ``smart.'' There is certainly a certain degree of overlap between this paper and our work, but also of important differences. In particular, the reasoning mechanisms we discussed here are not always applicable at the system level, which is the focus of \cite{thr19}. For instance, our techniques could be used in real-time by a CPS to determine whether its functional aspect is satisfied, but it may be unrealistic for a CPS to reason about its own trustworthiness. From another point of view, reasoning mechanisms discussed in \cite{thr19}, such as planning and decision-making, can be viewed as tools for the satisfaction of properties. In this sense, a designer might want to use the results of that paper to ensure that the decision-making mechanisms implemented \emph{within} a CPS satisfy certain properties that are responsible for ensuring the functional aspect of the CPS or even its trustworthiness.

The methodologies proposed in our paper build on a vast number of research results in ASP and related areas such as answer set planning, reasoning about actions, etc. and could be easily extended to deal with other aspects discussed in \cpsf{}. They are well-positioned for real-world applications given the efficiency and scalability of ASP-solvers (e.g., {\small \tt clingo}~\cite{GebserKNS07}) that can deal with millions of atoms, incomplete information, default reasoning, and features that allow ASP to interact with constraint solvers and external systems. 

\section{Conclusions and Future Work}
The paper presents a precise definition of a CPS, which, in conjunction with the CPS Ontology Framework by NIST, allows for the representing and reasoning of various problems that are of interest in the study of CPS. Specifically, the paper defines several problems related to the satisfaction of concerns of a CPS theories such as the problem of identifying non-compliant CPS systems, the problem of identifying the most/least trustworthy or vulnerable components, computing mitigation strategies, a most preferred mitigation strategies, or strategies with the best chance to succeed. For each problem, the paper presents a formal definition of ``\emph{what is the problem?}'' and provides an ASP program that can automatically verify such properties. To the best of our knowledge, all of these contributions are new to the research in Cyber-Physical Systems.  

The current ASP implementation\footnote{Available at \url{https://github.com/thanhnh-infinity/Research_CPS}} provides a first step towards developing a tool for CPS practitioners and designers. It automatically translates a system specification as an ontological description (e.g., as seen in Figure~\ref{fig:trust_tree_LoSC}) to ASP code and allows users to ask questions related to the aforementioned issues. It has been validated against small systems. One of our goals in the immediate near future is to develop an user-friendly interface that allows users to design or model their real-world CPS and identify potential issues within their systems and possible ways to address these issues before these issues become harmful.

\smallskip
\noindent 
{\footnotesize
{\bf Disclaimer.} 
Official contribution of the National Institute of Standards and Technology; not subject to copyright in the United States. Certain commercial products are identified in order to adequately specify the procedure; this does not imply endorsement or recommendation by NIST, nor does it imply that such products are necessarily the best available for the purpose.  Portions of this publication and research effort are made possible through the help and support of NIST via cooperative agreements 70NANB18H257 and 70NANB21H167.
}

\bibliographystyle{tlplike}
\bibliography{bibtex/enrico,bibtex/bibfile,bibtex/bib2010,bibtex/tnguyen2018,bibtex/biblio-mb,bibtex/biblio-mh}

\end{document}